\documentclass{article}
\usepackage{amsmath,epsfig,bm, amssymb,subfigure,graphicx,algorithm,algorithmic,color}
\newtheorem{defi}{Definition}
\newtheorem{prop}[defi]{Proposition}

\newcommand{\argmax}{\mathop{\mathrm{argmax\,}}}

\newcommand{\mathbbR}{\mathbb{R}}

\newcommand{\boldzero}{{\boldsymbol{0}}}
\newcommand{\boldone}{{\boldsymbol{1}}}

\newcommand{\boldF}{{\boldsymbol{F}}}
\newcommand{\boldG}{{\boldsymbol{G}}}

\newcommand{\boldI}{{\boldsymbol{I}}}

\newcommand{\boldK}{{\boldsymbol{K}}}
\newcommand{\boldL}{{\boldsymbol{L}}}
\newcommand{\boldM}{{\boldsymbol{M}}}

\newcommand{\boldQ}{{\boldsymbol{Q}}}
\newcommand{\boldR}{{\boldsymbol{R}}}

\newcommand{\boldX}{{\boldsymbol{X}}}

\newcommand{\boldc}{{\boldsymbol{c}}}

\newcommand{\boldf}{{\boldsymbol{f}}}

\newcommand{\boldr}{{\boldsymbol{r}}}

\newcommand{\boldu}{{\boldsymbol{u}}}

\newcommand{\boldx}{{\boldsymbol{x}}}
\newcommand{\boldy}{{\boldsymbol{y}}}

\newcommand{\boldalpha}{{\boldsymbol{\alpha}}}

\newcommand{\boldmu}{{\boldsymbol{\mu}}}

\newcommand{\boldGamma}{{\boldsymbol{\Gamma}}}

\newcommand{\boldSigma}{{\boldsymbol{\Sigma}}}

\newcommand{\calA}{{\mathcal{A}}}

\newcommand{\calI}{{\mathcal{I}}}

\newcommand{\calY}{{\mathcal{Y}}}

\usepackage{url,multirow}

\advance\oddsidemargin-1.0in
\date{\today}
\title{Ultra High-Dimensional Nonlinear Feature Selection for Big Biological Data} 

\author{Makoto Yamada$^{\ast,1,2}$ \and Jiliang Tang$^2$ \and Jose Lugo-Martinez$^3$ \and Ermin Hodzic$^4$ \and Raunak Shrestha$^5$ \and Avishek Saha$^2$ \and Hua Ouyang$^2$ \and Dawei Yin$^2$ \and Hiroshi Mamitsuka$^1$ \and Cenk Sahinalp$^{3,4}$ \and Predrag Radivojac$^3$ \and Filippo Menczer$^{2,3}$ \and Yi Chang$^2$ \\
$^1$Kyoto University. Gokasho, Uji, Japan \\
$^2$Yahoo Research, Sunnyvale, California, USA \\
$^3$Indiana University, Bloomington, Indiana, USA\\
$^4$Simon Fraser University, Burnaby, BC, Canada \\
$^5$Vancouver Prostate Centre, Vancouver, BC, Canada\\
\texttt{{makoto.m.yamada@ieee.org}}
}

\begin{document}
\maketitle

\begin{abstract}
 Machine  learning methods are used to discover complex nonlinear relationships in biological and medical data. However, sophisticated learning models are computationally unfeasible for data with millions of features. Here we introduce the first feature selection method for nonlinear learning problems that can scale up to large, ultra-high dimensional biological data. More specifically, we scale up the novel Hilbert-Schmidt Independence Criterion Lasso (HSIC Lasso) to handle millions of features with tens of thousand samples. The proposed method is guaranteed to find an optimal subset of maximally predictive features with minimal redundancy, yielding higher predictive power and improved interpretability. Its effectiveness is demonstrated through applications to classify phenotypes based on module expression in human prostate cancer patients and to detect enzymes among protein structures. We achieve high accuracy with as few as 20 out of one million features --- a dimensionality reduction of 99.998\%. Our algorithm can be implemented on commodity cloud computing platforms. The dramatic reduction of features may lead to the ubiquitous deployment of sophisticated prediction models in mobile health care applications.
\end{abstract}

\section{Introduction}
Life sciences are going through a revolution thanks to the possibility to collect and learn from massive biological data \cite{Marx:2013aa}. Efficient processing of such ``big data'' is extremely important for many medical and biological applications, including disease classification, biomarker discovery, drug development and scientific discovery \cite{guyon2003introduction}.
The complexity of biological data is dramatically increasing due to improvements in measuring devices such as next-generation sequencers, microarrays and mass spectrometers \cite{Li2014187}.
As a result, we must deal with data that includes many observations (hundreds to tens of thousands) and even larger numbers of features (thousands to millions). Machine learning algorithms are charged with learning patterns and extracting actionable information from biological data. These techniques have been used successfully in various analytical tasks, such as genome-wide association studies \cite{burton2007genome} and gene selection \cite{guyon2003introduction}.

However, the scale and complexity of big biological data pose new challenges to existing machine learning algorithms. There is a trade-off between scalability and complexity: linear methods scale better to large data, but cannot model complex patterns. Nonlinear models can handle complex relationships in the data but are not scalable to the size of current datasets. In particular, learning nonlinear models requires a number of observations that grows exponentially with the number of features \cite{guyon2003introduction,saeys2007review}. Biological data generated by modern technology has as many as millions of features, making the learning of nonlinear models unfeasible with existing techniques. To make matters worse, current nonlinear approaches cannot take advantage of distributed computing platforms.

A promising approach to make nonlinear analysis of big biological data computationally tractable is to reduce the number of features. This method is called \emph{feature selection} \cite{guyon2003introduction}. Biological data is often represented by matrices where rows denote features and columns denote observations. Feature selection aims to identify a subset of features (rows) to be preserved, while eliminating all others. There are two reasons why the predictive capability of the data may be preserved or even improved when many features are excluded. First, measurements generate many features automatically and their quality is hard to control \cite{donoho2008higher}. Second, biological features are often highly redundant, so that the number of useful features is small. For example, among millions of Single Nucleotide Polymorphisms (SNPs), only a few are useful to predict a certain disease \cite{saeys2007review}. Although state-of-the-art feature selection algorithms, such as \emph{minimum redundancy maximum relevance} (mRMR) \cite{PAMI:Peng+etal:2005}, have been proven to be effective and efficient in preparing data for many tasks, they cannot scale up to biological data with millions of features. Moreover, mRMR uses greedy search strategies such as forward selection/backward elimination and tends to produce locally optimal feature sets.
 
 {
Here we propose a novel feature selection framework for big biological data that makes it possible for the first time to identify very few relevant, non-redundant features among millions. The proposed method is based on two components: \emph{Least Angle} Regression (LARS), an efficient feature selection method \cite{efron2004least}, and the Hilbert-Schmidt Independence Criterion (HSIC), which enables the selection of features that are non-linearly related~\cite{ALT:Gretton+etal:2005}. These properties are combined to obtain a method that can exploit \emph{nonlinear} feature dependencies efficiently, and furthermore enables \emph{distributed} implementation on commodity cloud computing platforms. We name our algorithm \emph{Least Angle Nonlinear Distributed} (LAND) feature selection. Experiments demonstrate that the proposed method can reduce the number of features in real-world biological datasets from one million to tens or hundreds, while preserving or increasing prediction accuracy for biological applications. }

The following sections present the proposed LAND method in detail and show performance evaluation of LAND on three large, high-dimensional datasets related to the problems of discovering mutations in the tumor suppressor protein p53, classifying cancer phenotypes in a cohort of human prostate cancer patients, and detecting enzymes among protein structures. While existing feature selection methods cannot consider nonlinear dependencies among features in these problems, we show that our approach can reduce the dimensionality by five orders of magnitude.

These results are achieved in minutes to hours of cluster computing time. The selected features are relevant and non-redundant, making it possible to obtain accurate and interpretable models that can be run on a laptop computer.

\vspace{.1in}
\noindent {\bf Contribution:}
\begin{itemize}
\item {We scale up the novel Hilbert-Schmidt Independence Criterion Lasso (HSIC Lasso)~\cite{yamada2014high} to handle ultra high-dimensional and large-scale datasets. To the best of our knowledge, this is the first \emph{minimum redundancy maximum relevance} feature selection method that can handle tens of thousand data samples with millions of features.}
\item {We propose the first implementation of nonlinear feature selection on distributed computing platforms.}
\item {We demonstrate that LAND feature selection outperforms state-of-the-art methods on three real-world, large, high-dimensional datasets.}
\end{itemize}

\section{Related Work}
\label{sec:existing}

In this section, we review existing nonlinear feature selection methods and show their drawbacks.

 Maximum Relevance (MR) feature selection is a popular approach that selects $m$ features with the largest relevance to the output \cite{PAMI:Peng+etal:2005}. The feature screening method \cite{balasubramanian2013ultrahigh} is also an MR-method. Usually, the mutual information and a kernel-based independence measures such as HSIC are used as the relevance score \cite{ALT:Gretton+etal:2005}. MR-based methods are simple yet efficient and can be easily applicable to high-dimensional and large sample problems. However, since MR-based approaches only use input-output relevance and not use input-input relevance, they tend to select redundant features (i.e., the selected features can be very similar to each other). As a result it may not help in improving overall classification/regression accuracy and interpretability. 

Minimum Redundancy Maximum Relevance (mRMR)~\cite{PAMI:Peng+etal:2005} was proposed to deal with the feature redundancy problem; it selects features that have high relevance with respect to an output and are non-redundant. It has been experimentally shown that mRMR outperforms MR feature selection methods~\cite{PAMI:Peng+etal:2005}. Moreover, there exists an off-the-shelf C++ implementation of mRMR, and it can be applicable to a large and high dimensional feature selection. Fast Correlation based filter (FCBF) can also be regarded as an mRMR method, in which it uses symmetrical uncertainty to calculate
dependences of features and finds best subset using backward selection with sequential search strategy \cite{yu2003feature}. Note that, it has also been reported that FCBF compares favorably with mRMR \cite{senliol2008fast}. However, both mRMR and FCBF use greedy search strategies such as forward selection/backward elimination and tends to produce locally optimal feature set. 

To obtain a globally optimal feature set, a convex relaxed version of mRMR called the Quadratic Programming Feature Selection (QPFS) and SPEC were proposed in~\cite{JMLR:Rodriguez+etal:2010,nguyen2014effective}. An advantage of QPFS and SPEC over mRMR is that it can find a globally optimal solution by just solving a QP problem. The authors showed that QPFS compares favorably with mRMR for large sample size but low-dimensional cases (e.g., $d < 10^3$ and $n > 10^4$). However, QPFS and SPEC tend to be computationally expensive for large and high-dimensional  cases, since they need to compute $d(d -1)/2$ mutual information scores. To deal with the computational problem in QPFS, a Nystr\"{o}m approximation based approach was proposed~\cite{JMLR:Rodriguez+etal:2010}, and it has been experimentally shown that QPFS with Nystr\"{o}m approximation compares favorably with mRMR both in accuracy and time.  However, for large and high-dimensional settings, computational cost for mutual information is still very high. 

Forward/Backward elimination based feature selection with HSIC (FOHSIC/BAHSIC) is also a widely used feature selection method \cite{song2012feature}.   An advantage of HSIC-based feature selection over mRMR is that the HSIC score can be accurately estimated. Moreover, HSIC can be implemented very easily. However, similar to mRMR, it selects features using greedy search algorithm and tends to have a locally optimal feature set. To obtain a better feature set, HSFS was proposed~\cite{DBLP:conf/icml/MasaeliFD10} as a continuously relaxed version of FOHSIC/BAHSIC that could be solved by limited-memory BFGS (L-BFGS)~\cite{book:Nocedal:2003}.  However, HSFS is a non-convex method and restarting from many different initial points would be necessary to select good features which is computationally expensive. 

For small and high-dimensional feature selection problems (e.g., $n < 100$ and $d> 10^4$), $\ell_1$ regularized based approaches such as Lasso are useful \cite{JRSSB:Tibshirani:1996,zhao2010efficient}. In addition, Lasso is known to scale well with both number of samples as well as dimensionality \cite{JRSSB:Tibshirani:1996,zhao2010efficient}. However, Lasso can only capture linear dependency between input features and output values. To handle non-linearity,  { HSIC Lasso} was proposed recently \cite{yamada2014high}. In HSIC Lasso, with specific choice of kernel functions, non-redundant features with strong statistical dependence on output values can be found in terms of HSIC by simply solving a Lasso problem. Although, empirical evidence~\cite{yamada2014high} shows that HSIC Lasso outperforms most existing feature selection methods, in general HSIC Lasso tends to be expensive compared to simple Lasso when the number of samples increases. Moreover, statistical properties of HSIC Lasso is not well studied. Recently, a few \emph{wrapper} type of feature selection methods including the feature generating machine \cite{JMLR:v15:tan14a} and SVM based approach \cite{sunfeature} were proposed. These methods are state-of-the-art feature selection methods for high-dimensional and/or large-scale data. However, those wrapper methods are computationally expensive for \emph{ultra} high-dimensional and large-scale datasets. 

Sparse Additive Models (SpAM) are useful for high-dimensional feature selection problems~\cite{ravikumar2009sparse,NIPS2008_0329, raskutti2012minimax,suzuki2012fast} and	can be efficiently solved by the \emph{back-fitting} algorithms~\cite{ravikumar2009sparse} resulting in globally optimal solutions. Also, statistical properties of the SpAM estimator are well studied~\cite{ravikumar2009sparse}.  However, a potential weakness of SpAM is that it can only deal with additive models and may not work well for non-additive models. 
 Hierarchical Multiple Kernel Learning (HMKL) \cite{bach2009exploring,jawanpuria2015generalized} is also a nonlinear feature selection method and can fit complex functions such as non-additive functions. However, the computation cost of HMKL is rather expensive. In particular, since HMKL searches the combination of kernels from $(m + 1)^d$ combinations ($m$ is the total number of selected kernels), the computation cost heavily depends on the dimensionality $d$.

\section{Least Angle Nonlinear Distributed  feature  selection}
We first formulate the supervised feature selection problem and propose the \emph{Least Angle Nonlinear Distributed} (LAND) feature selection. 
\subsection{Problem Formulation}
Let $\boldX = [\boldx_1, \ldots, \boldx_n] = [\boldu_1, \ldots, \boldu_d]^\top \in \mathbbR^{d \times n}$ denotes the input data, a matrix where a column $\boldx_i \in \mathbbR^d$ represents an observation vector composed of $d$ elements (features) and a row $\boldu_j \in \mathbbR^n$ represents a feature vector composed of $n$ elements (observations).  
Let $\boldy = [y_1, \ldots, y_n]^\top \in \mathbbR^{n}$ denotes the output data or labels
so that $y_i \in \calY$ is the label for $\boldx_i$.
The output domain $\calY$ can be either continuous (as in regression problems) or categorical (as in classification problems). 

The goal of supervised feature selection is to find $m$ features ($m \ll d$) that are most relevant for predicting the output  $\boldy$ for observations $\boldX$.

To efficiently solve a large and high-dimensional feature selection problem, next we propose a \emph{nonlinear} extension of LARS~\cite{efron2004least} leveraging HSIC~\cite{ALT:Gretton+etal:2005}. Then, we introduce an approximation to reduce the memory and computational requirements of the algorithm. This approximation enables our feature selection method to be deployed on a distributed computing platform, scaling up to big biological data. 

\subsection{{HSIC Lasso with Least Angle Regression}} 

Let us define the kernel (similarity) matrix of the $k$-th feature observations
\begin{align*}
[\boldK^{(k)}]_{ij} &= K(u_{ki}, u_{kj}),\hspace{0.3cm} i,j = 1, \ldots, n,
\end{align*}
and outputs
\begin{align*}
[\boldL]_{ij} &= L(y_i,y_j),~~\hspace{0.3cm}i,j=1, \ldots, n,
\end{align*}
where $u_{ki}$ is the $i$-th element of $k$-th feature vector $\boldu_k$ and $K(u,u')$ and $L(y,y')$ are kernel functions. In principle, any universal kernel function such as the Gaussian or Laplacian kernels can be used \cite{ALT:Gretton+etal:2005}. Here, we first normalize feature $\boldu$ to have unit standard deviation and then use the Gaussian kernel \[
K(u,u') = \exp \left(-\frac{(u - u' )^2}{2\sigma_{\mathrm u}^2} \right),
\]
where $\sigma_\mathrm{u}$ is the kernel width. 

For the outputs, in regression cases ($y \in \mathbbR$) we similarly normalize $y$ to have unit standard deviation and then use the Gaussian kernel \[
L(y,y') = \exp \left(-\frac{( y - y' )^2}{2\sigma_{\mathrm y}^2} \right).
\]
In this paper, we use $\sigma_{\mathrm{u}} =1$ and $\sigma_{\mathrm{y}} = 1$. 
In classification cases (i.e., $y$ is categorical) we use the delta kernel, which has been shown to be useful for multi-class problems~\cite{song2012feature}: 
\begin{eqnarray*}
L(y,y') = \left\{ \begin{array}{ll}
{1}/{n_{y}} & \textnormal{if}\hspace{0.3cm} y = y' \\
0 & \textnormal{otherwise}, \\
\end{array} \right.
\end{eqnarray*} 
 where $n_{y}$ is the number of observations in class $y$.   

LAND {(HSIC Lasso)} is formulated as \cite{yamada2014high}
\begin{align}
\label{eq:land}
\begin{split}
\min_{\boldalpha \in \mathbbR^d} & \hspace{0.3cm}
\left\|\widetilde{\boldL}\! -\! \sum_{k = 1}^{d} \alpha_k \widetilde{\boldK}^{(k)} \right\|^2_{\textnormal{F}} + \lambda \|\boldalpha\|_1, \\
\textnormal{s.t.} & \hspace{0.3cm} \alpha_1,\ldots,\alpha_d \geq 0,
\end{split}
\end{align}
where $\lambda \geq 0$ is a regularization parameter, $\|\cdot\|_1$ is the $\ell_1$ norm, $\|\cdot\|_F$ is the Frobenius norm ($\|\boldM\|_{\textnormal{F}} = \sqrt{\text{tr}(\boldM\boldM^\top)}$), and $\widetilde{\boldK}$ and $\widetilde{\boldL}$ are the normalized kernel matrices such that $\boldone_n^\top \widetilde{\boldK}\boldone_n = \boldone_n^\top \widetilde{\boldL}\boldone_n = 0$ and $\|\widetilde{\boldK}\|_{\text{F}}^2 = \|\widetilde{\boldL}\|_{\text{F}}^2 = 1$. {In this paper, we employ the least angle regression \cite{efron2004least} to solve Eq. \eqref{eq:land} (See Algorithm \ref{alg:land}), and we name the LARS variant of Eq. \eqref{eq:land} as \emph{Least Angle Nonlinear Distributed} (LAND). }  
  
The solution of the LAND problem enables the selection of the most relevant, least redundant features. To illustrate why, 
we can rewrite the objective function in Eq.\eqref{eq:land} as: 
\begin{align}
1 - 2\sum_{k = 1}^d\alpha_k {\textnormal{NHSIC}}(\boldu_k,\boldy) + \sum_{k,l = 1}^d \alpha_k \alpha_l {\textnormal{NHSIC}}(\boldu_k,\boldu_l),
\end{align}
where $\text{NHSIC}(\boldu, \boldy) = \text{tr}(\widetilde{\boldK} \widetilde{\boldL})$ is the normalized version of HSIC~\cite{DBLP:journals/jmlr/CortesMR12}, an independence measure such that ${\textnormal{NHSIC}}(\boldu,\boldy) = 1$ if $\boldu = \boldy$ and ${\textnormal{NHSIC}}(\boldu,\boldy) = 0$ if and only if the two random variables $\boldu$ and $\boldy$ are independent (see proof in Section~\ref{sec:proof_full}). In the original paper \cite{yamada2014high}, the un-normalized HSIC was employed. However, since the un-normalized HSIC takes some positive number when input and output variables are dependent. That is, $\text{HSIC}(\boldu_k,\boldy) > \text{HSIC}(\boldu_{k'},\boldy)$ may not mean that $\boldu_k$ is more highly associated with $\boldy$ than $\boldu_{k'}$. Thus, it is natural to normalize HSIC for feature selection problems.

If output $\boldy$ has high dependence on the $k$-th feature $\boldu_k$,   ${\textnormal{NHSIC}}(\boldu_k,\boldy)$ is large and thus $\alpha_k$ should also be large, meaning that the feature should be selected. On the other hand, if $\boldu_k$ and $\boldy$ are independent, ${\textnormal{NHSIC}}(\boldu_k,\boldy)$ is close to zero; $\alpha_k$ should thus be small and the $k$-th feature will not be selected. Furthermore, if $\boldu_k$ and $\boldu_l$ are strongly dependent on each other, ${\textnormal{NHSIC}}(\boldu_k,\boldu_l)$ is large and thus either $\alpha_k$ or $\alpha_l$ will be small; only one of the redundant features will be selected. 

In practice, LAND iteratively selects non-redundant features with a strong relevance for determining the output. To select the $k$-th feature we first need to consider its relevance with respect to the output, indicated by $\text{NHSIC}(\boldu_k, \boldy)$. Second, a feature is discounted based on its redundancy with respect to previously selected features, given by $\sum_{i: \alpha_i > 0} \alpha_i \text{NHSIC}(\boldu_k, \boldu_i)$\sloppy. Hence we define the \emph{selection score} of the $k$-th feature as $c_k = \text{NHSIC}(\boldu_k, \boldy) - \sum_{i: \alpha_i > 0} \alpha_i \text{NHSIC}(\boldu_k, \boldu_i)$. After the feature is selected, we update the $\alpha$ coefficients. 

A key challenge of solving problem~\eqref{eq:land} is that it requires huge memory ($O(dn^2)$) to store all kernel matrices $\widetilde{\boldK}^{(k)}$. For example in the enzyme dataset described below ($d=$1,062,420, $n=$15,328), the naive implementation 
requires more than a petabyte of memory, which is not feasible. Moreover, the computing time for matrix multiplications scales as $O(mdn^3)$, making it unfeasible when both $d$ and $n$ are large. We address these issues by applying a kernel approximation.  

\subsection{Nystr\"{o}m Approximation for NHSIC}

The Nystr\"{o}m approximation~\cite{book:Schoelkopf+Smola:2002} allows us to rewrite $\text{NHSIC}(\boldu,\boldy) = \text{tr}(\widetilde{\boldK}\widetilde{\boldL}) \approx \text{tr}(\boldF\boldF^\top \boldG\boldG^\top)$, where, in the regression case, 
\begin{align*}
\boldF &= \boldGamma \boldK_{nb}\boldK_{bb}^{-1/2}/(\text{tr}((\boldK_{bb}^{-1/2}\boldK_{nb}^\top\boldK_{nb} \boldK_{bb}^{-1/2})^2))^{1/4}, \\
\boldG &= \boldGamma \boldL_{nb}\boldL_{bb}^{-1/2}/(\text{tr}((\boldL_{bb}^{-1/2}\boldL_{nb}^\top \boldL_{nb}\boldL_{bb}^{-1/2})^2))^{1/4}.
\end{align*}
Here, $\boldF\boldF^\top$ is a low-rank approximation of $\widetilde{\boldK}$, such that $\boldK_{nb} \in \mathbbR^{n \times b}$ and $[\boldK_{nb}]_{ij} = K(u_{i},u_{b,j})$, where $\boldu_b \in \mathbbR^b$ is a basis vector (See the experimental section for more details). Analogously,  $\boldK_{bb} \in \mathbbR^{b \times b}$, $\boldL_{nb} \in \mathbbR^{n \times b}$, and $\boldL_{bb} \in \mathbbR^{b \times b}$. The parameter $b$ is an upper bound on the rank of the $\boldK_{nb}$ and $\boldL_{nb}$ matrices. The higher $b$, the better the approximation, but the higher the computational and memory costs. If the number of observations $n$ is very large, we can make the problem tractable by using $b \ll n$ without sacrificing the predictive power of the selected features, as shown the next section. 
The resulting complexity of kernel computation and multiplication for selecting $m$ features is $O(dbn + mdb^2n) = O(mdb^2n)$. Moreover, for each dimension, we only need to store the $\boldF \in \mathbbR^{b\times n}$ matrix, yielding space complexity $O(dbn)$. The approximation reduces the overall time complexity of the algorithm by a factor $O(n^2/b^2)$ and the memory requirements by a factor $O(n/b)$ (see Table~1). Note that, the original LAND formulation Eq.~\eqref{eq:land} can find a globally optimal solution without the Nystr\"{o}m approximation. However, in practice, since each kernel Gram matrix (i.e., $\boldK^(k)$) is computed from only one feature, we can accurately approximate each kernel Gram matrix by the Nystr\"{o}m approximation. Thus, we can empirically find a good solution if we set the number of bases in the Nystr\"{o}m approximation relatively large (in this paper, we found b = 10, 20 works well).

\begin{table}
\centering
\caption{Summary of computational complexity and memory size of different implementations of LAND. The total time complexity is obtained by adding the kernel computation and multiplication times, which are dominated by the latter.\label{tab:complexity}}
\begin{tabular}{lccc}
\hline
Method & Kernel & Multiplication  & Memory  \\ \hline
Na\"{i}ve &  $O(dn^2)$ & $O(mdn^3)$  & $O(dn^2)$ \\ 
Nystr\"{o}m  & $O(dbn)$  &  $O(mb^2dn)$ &  $O(dbn)$\\ 
MapReduce & $O(dbn/M)$  &  $O((mb^2/M)dn)$ & $O(dbn)$\\
\hline
\end{tabular}
\end{table}


In the classification case, we can use the above technique to approximate the kernel matrix $\boldF$ and compute $\boldG$ as
\begin{align*}
\boldG_{k,j} &= \left\{ \begin{array}{ll}
\frac{1}{\sqrt{n_k}}& \text{if}\hspace{0.3cm}k = y_{j} \\
0& \text{otherwise},\\
\end{array} \right. 
\end{align*}
where $\boldG \in \mathbbR^{C \times n}$ and $C$ is the number of classes.  The computational complexity of kernel computation and multiplication is $O(mbdn(b+C))$ and the memory complexity is $O(dn(b+C))$. These too are dramatic reductions in computational time and memory.  


\subsection{Distributed computation}

While the Nystr\"{o}m approximation is useful for data with many observations (large $n$), the computational cost of LAND makes it unfeasible on a single computer for ultra high-dimensional cases, i.e., when the number of features is extremely large (e.g., $d \geq 10^6$). Fortunately,  we can compute the kernel matrices $\{\boldF_k\}_{k = 1}^d$ in parallel. The selection scores $c_k$ can be computed independently as well. These properties make it possible to further speed up LAND with a distributed computing framework. 
The resulting  computational complexity is $O(mdb^2n/M)$, where $M$ is the number of mappers (Table~\ref{tab:complexity}). 


The proposed LAND algorithm is implemented on a cluster for scalability to large datasets.  Map-Reduce is a widely adopted distributed computing framework. It consists of a \emph{map} procedure that breaks up the problem into many small tasks that can be performed in parallel, and distributes these tasks to multiple computing nodes (\emph{mappers}). A \emph{reduce} procedure then is executed on multiple nodes (\emph{reducers}) to aggregate the computed results. 
For example, we denote the map function with inputs $\{\boldF_k\}_{k = 1}^d, \boldG$ and the corresponding reduce function as
\begin{align*}
&\text{map}(\{\boldF_k\}_{k = 1}^d, \boldG : \, \langle k, \text{tr}((\boldF_k^\top \boldG)^2 ) \rangle ) \\
&\text{reduce}( \langle k, \, \text{tr}((\boldF_k^\top \boldG)^2 ) \rangle : \, f_k = \text{tr}((\boldF_k^\top \boldG)^2) )
\end{align*}
where the map function returns key-value pairs and the reduce function stores the key-value pairs into a vector $f_k$. 

We employ two Map-Reduce frameworks. Hadoop (\url{hadoop. apache.org}) is used for computing $\boldF_k$'s in the Nystr\"{o}m approximation. Spark (\url{spark.apache.org}) reduces the data access cost by storing intermediate results in memory, and is used for the iteration operations. 
Our Map-Reduce implementation is shown in Algorithm 1. 

We use $b = 20$ for the p53 data and the prostate cancer data, and $b = 10$ for the enzyme data.

\subsection{Relation to high-dimensional feature screening method}
\label{sec:proof_full}
Let us establish a relation between the proposed method, LAND, and the feature screening method~[2]. 
Feature screening is a maximum relevance~[15] approach  used widely in the statistics community.
It aims to select a subset of features with the goal of dimensionality reduction, without affecting the statistical properties of the data. 
The idea is to rank the covariates between the input variables $\boldu$ and the output response $\boldy$ according to some degree of dependence. For example, one can choose $\text{NHSIC}$ as an independence measure and rank the $d$ features $\boldu_1,\ldots,\boldu_d$ according to the values of $\text{NHSIC}(\boldu_k,\boldy)$. The top $m$ features are then selected. The MR-NHSIC baseline can be regarded as a feature screening method.



\begin{prop}
If any pair of features $\boldu_k$ and $\boldu_{k'}$ are assumed to be independent, then there exists a pair $(\lambda, m)$ such that the top $m$ features obtained by the feature screening method~[2] 
are the same of those obtained by solving Eq.~\ref{eq:land}.
\end{prop}

\begin{proof}
According to 
Theorem 4 by Gretton et al.~[14], $\text{HSIC}(\boldu_k, \boldu_{k'})=0$ if and only if two features $\boldu_k$ and $\boldu_{k'}$ are independent. Hence, if the pair of features  $\boldu_k$ and $\boldu_{k'}$ is  independent, we have the following result using the definition of NHSIC: 
\begin{align}
\label{eq:NHSIC_IND}
\text{NHSIC}(\boldu_k, \boldu_{k'}) &= \text{tr}(\bar{\boldK}^{(k)} \bar{\boldK}^{(k')}) \nonumber \\&
= \frac{\text{HSIC}(\boldu_k,\boldu_{k'})}{\sqrt{\text{tr}(\bar{\boldK}^{(k)}\bar{\boldK}^{(k)} )} \sqrt{\text{tr}(\bar{\boldK}^{(k')} \bar{\boldK}^{(k')})}} \nonumber \\
&= \left\{ 
\begin{array}{l}
0~~~~~\text{if $k \neq k'$} \\
1~~~~~\text{if $k = k'$.}
\end{array}
\right.
\end{align}
Since the two features $\boldu_k$ and $\boldu_{k'}$ are assumed to be independent (i.e., \linebreak $\|\sum_{k = 1}^{d} \alpha_k \widetilde{\boldK}^{(k)}\|^2_{\text{F}} = \|\boldalpha\|_2^2$) and $\|\bar{\boldL}\|^2_{\text{F}} = \text{NHSIC}(\boldy,\boldy) =  1$ by the definition of NHSIC (Eq.~\eqref{eq:NHSIC_IND}), the optimization problem in Eq.~\eqref{eq:land} is equivalent to:
\begin{align}
\label{eq:prop}
\begin{split}
\max_{\boldalpha \in \mathbbR^{d}} &\hspace{0.3cm} \sum_{k=1}^d \alpha_k \text{NHSIC}(\boldu_k, \boldy) - \frac{1}{2}\|\boldalpha\|_2^2 - \frac{\lambda}{2} \|\boldalpha\|_1,\\
\text{s.t.} &\hspace{0.3cm} \alpha_1,\ldots,\alpha_d \geq 0. 
\end{split}
\end{align} 
Next we prove by contradiction that the largest NHSIC values correspond to the largest $\alpha_k$ values in the solution of Eq.~\ref{eq:prop}. Suppose there exists a pair $(i,j)$ such that $\text{NHSIC}(\boldu_i,\boldy) > \text{NHSIC}(\boldu_j,\boldy)$ and $\alpha_i < \alpha_j$. Then one can simply switch the values of $\alpha_i$ and $\alpha_j$ to obtain a higher value in the objective function of Eq.~\ref{eq:prop}. This contradiction proves that the largest $\alpha_k$ correspond to the largest values of $\text{NHSIC}(\boldu_k,\boldy)$.
\end{proof}

The above proposition draws the connection to high-dimensional feature screening~[2]. 
%
Since the feature screening method tends to select redundant features, an iterative screening approach is used to filter out redundant features. 


\begin{algorithm}[t]
	\caption{LAND (MapReduce Spark version)} 
    \label{alg:land}
	\begin{algorithmic}
		\STATE {\bf Initialize}: $\boldalpha = \boldzero_{\text{d}}$, $\calA = []$ (active set), and $\calI = \{1, 2, \ldots, d\}$ (non-active set).
		\STATE Compute $\boldG$ and store it memory.
		\STATE /*Compute $\{\boldF_k\}_{k = 1}^d$ and store them in memory.*/
		\STATE $\{\boldF_k\}_{k = 1}^d = \text{map}(\{\boldu_k\}_{k=1}^d: <k, \boldF_k>)$
		\STATE
		\STATE /* Compute $\text{NHSIC}(\boldu_k, \boldy)$*/
		\STATE $\text{map}(\{\boldF_k\}_{k = 1}^d, \boldG : <k, \text{tr}(\boldF_k^\top \boldG)^2>)$
		\STATE $\text{reduce}(<k, \text{tr}\left((\boldF_k^\top \boldG)^2\right)>: f_k = \text{tr}\left((\boldF_k^\top \boldG)^2\right))$
		\STATE
		\STATE $\boldR = []$
		\STATE /* Select $m$ features */
		\WHILE{$|\calA| < m$}
		\STATE /*Compute $c_k = \text{NHSIC}(\boldu_k, \boldy) - \sum_{i = 1}^d \alpha_i \text{NHSIC}(\boldu_k, \boldu_i)$ */
		\STATE $\boldc = \boldf - \boldR \boldalpha_{\calA}$
		\STATE
		\STATE {\bf Find feature index}: $j = \argmax_{\boldc_{\calI}} c_k > 0$
		\STATE {\bf Update sets}: $\calA = [\calA\hspace{0.3cm} j], \calI = \calI \backslash j$
		\STATE {\bf Update coefficients}:
			\begin{align*}
			\boldalpha_{\calA} &= \boldalpha_{\calA} + \widehat{\mu} \boldQ_{\calA}^{-1} \boldone, \\
			[\boldQ_{\calA}]_{i,j} &= \text{NHSIC}(\boldu_{\calA,i}, \boldu_{\calA,j})\\
			\widehat{\mu} &= \min_{\mu} \left\{ \begin{array}{ll}
			\exists \ell \in \calI: \widetilde{\boldc}_\ell = \boldc_{\calA}& \\
			\boldc_{\calA} = \boldzero& \\
			\end{array} \right.,
			\end{align*}
               	\STATE /*Compute $\{\text{NHSIC}(\boldu_j, \boldu_k)\}_{k = 1}^d$ */
		\STATE $\text{map}(\{\boldF_k\}_{k = 1}^d, \boldF_j : <k, \text{tr}(\boldF_k^\top \boldF_j)^2>)$
		\STATE $\text{reduce}(<k, \text{tr}(\boldF_k^\top \boldF_j)^2>: r_{k,j} =\text{tr}(\boldF_k^\top \boldF_j)^2>)$
		\STATE $\boldR = [\boldR \hspace{0.3cm} \boldr_j]$
               \ENDWHILE
	\end{algorithmic} 
\end{algorithm}
  
\begin{figure}[h]
\begin{center}
\begin{minipage}[t]{0.475\linewidth}
\centering
  {\includegraphics[width=0.99\textwidth]{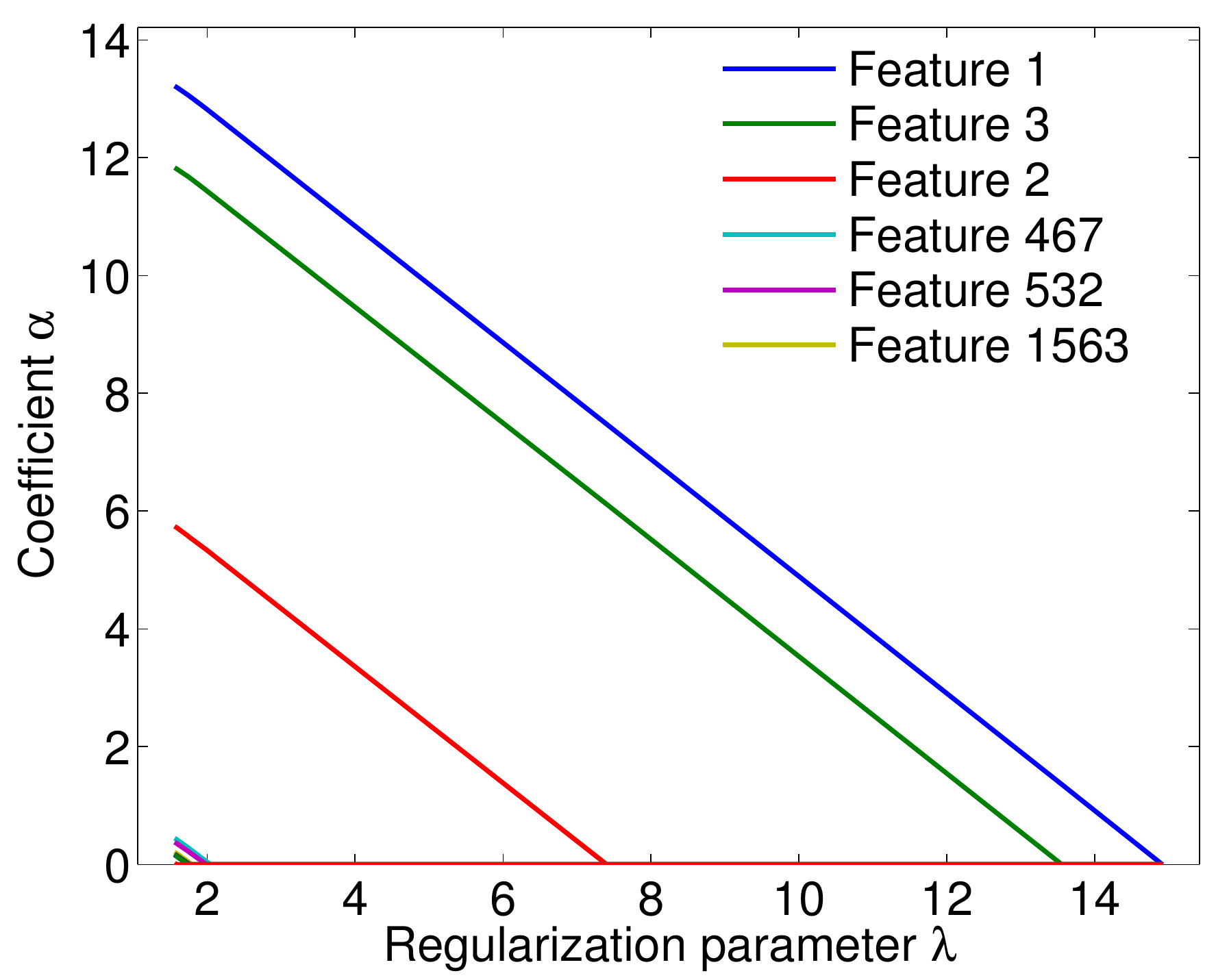}} 
  (a) Regularization path \\ \vspace{-0.10cm}
\end{minipage}
\begin{minipage}[t]{0.49\linewidth}
\centering
  {\includegraphics[width=0.99\textwidth]{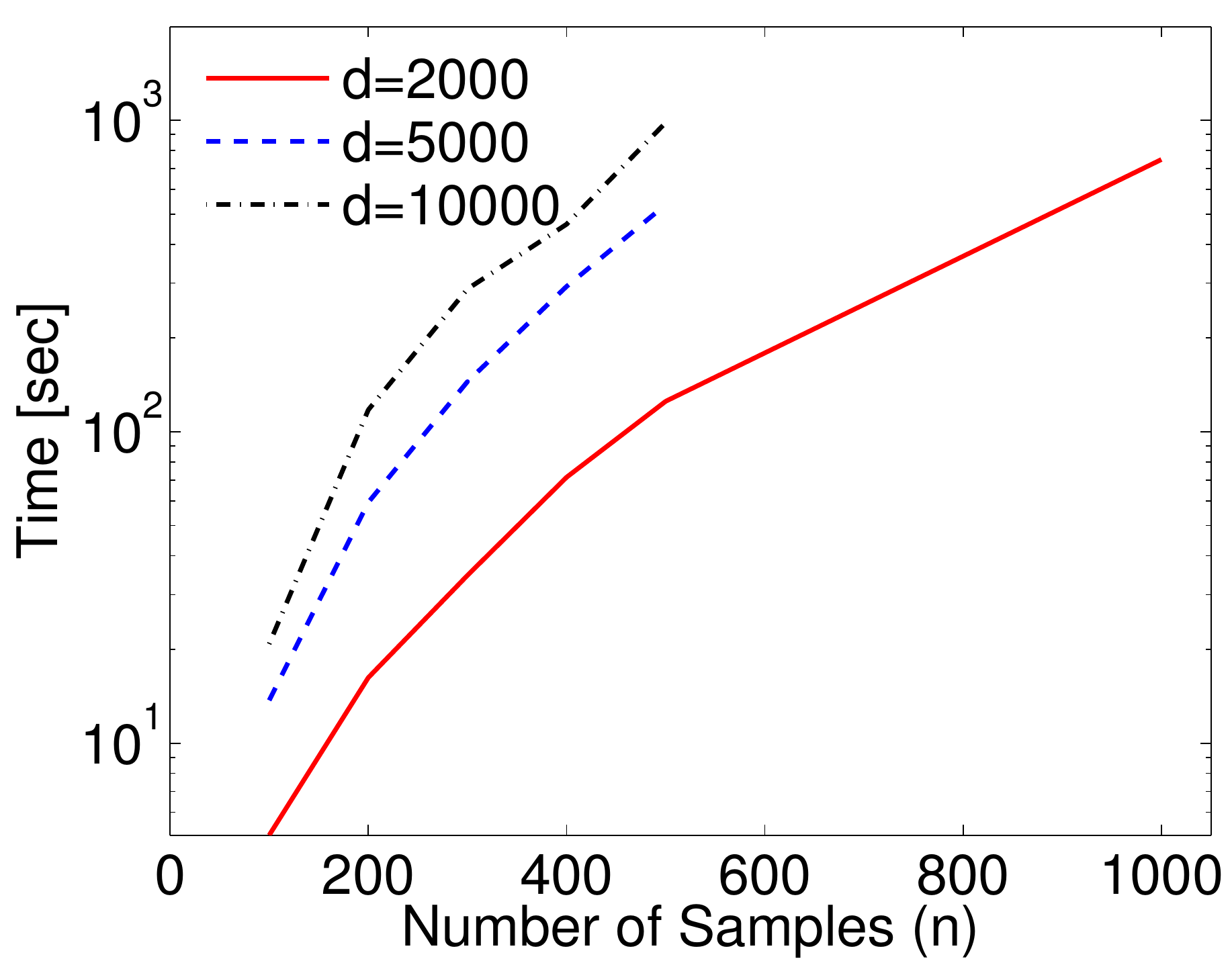}} 
  (b) Computational time \\ \vspace{-0.10cm}
\end{minipage}
\caption{Performance of LAND on synthetic data. (a) The regularization path, which describes the transition of parameters over the regularization parameter $\lambda$ in Eq.~\ref{eq:land}. (b) Computational time (without using the Nystr\"{o}m approximation and the distributed implementation) versus the number of observations $n$, for different values of the dimensionality $d$.}  
    \label{fig:illustrative_example}
\end{center}
\end{figure}

\begin{figure*}[t]
\begin{center} 
\begin{minipage}[t]{0.245\linewidth}
\centering
  {\includegraphics[width=0.99\textwidth]{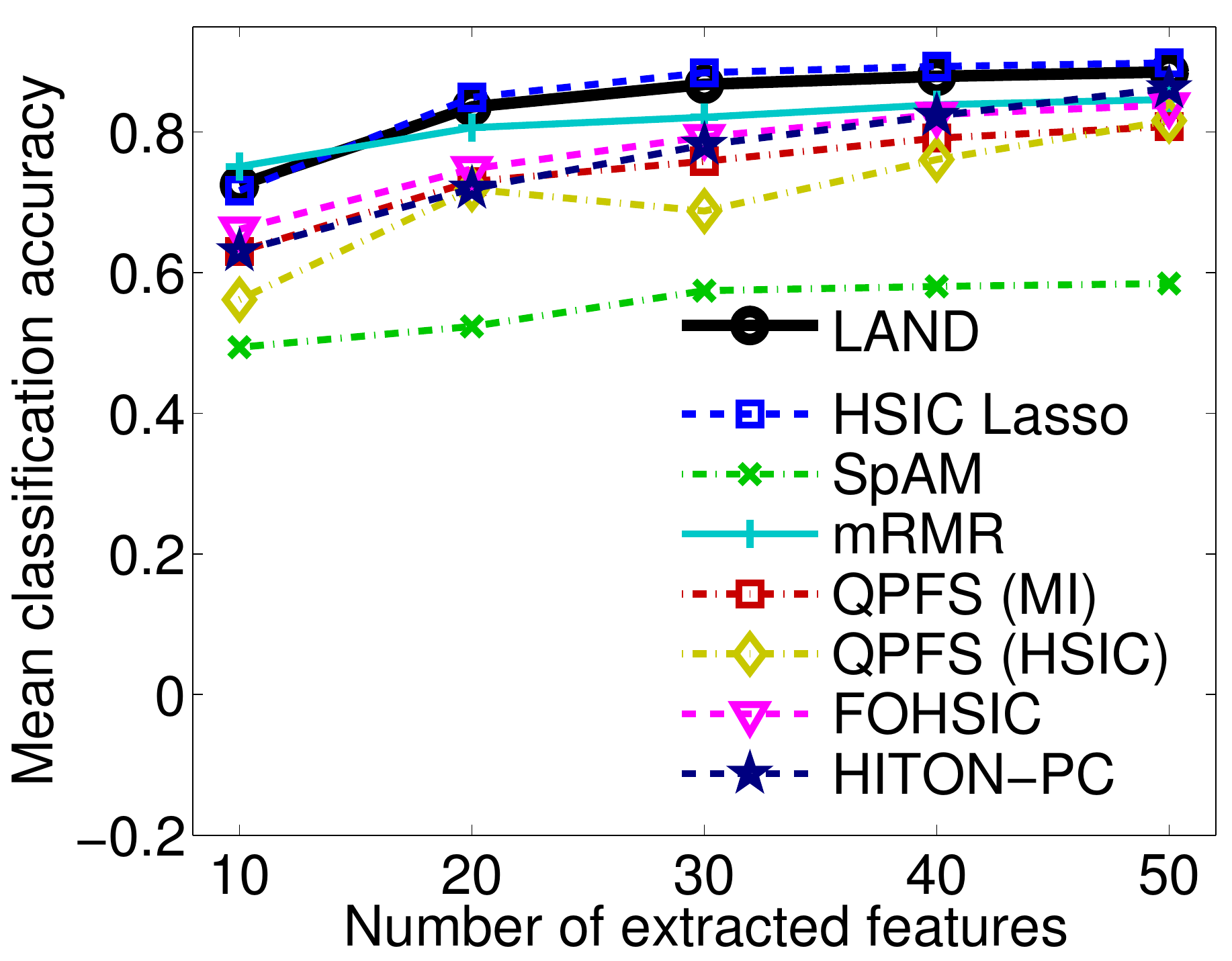}} \\ \vspace{-0.10cm}
(A) AR10P
\end{minipage}
\begin{minipage}[t]{0.245\linewidth}
\centering
  {\includegraphics[width=0.99\textwidth]{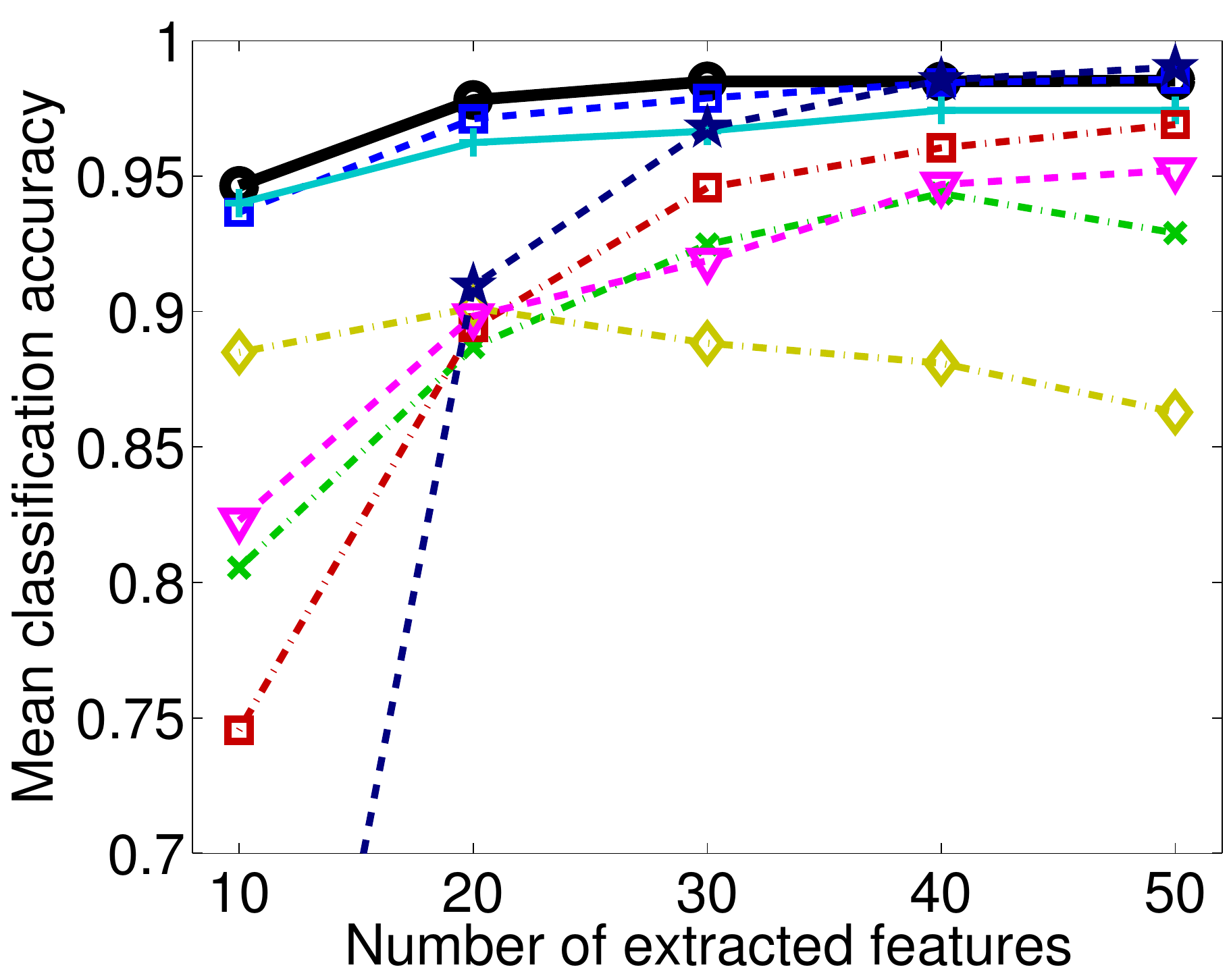}} \\ \vspace{-0.10cm}
(B) PIE10P
\end{minipage}
  \begin{minipage}[t]{0.245\linewidth}
\centering
{\includegraphics[width=0.99\textwidth]{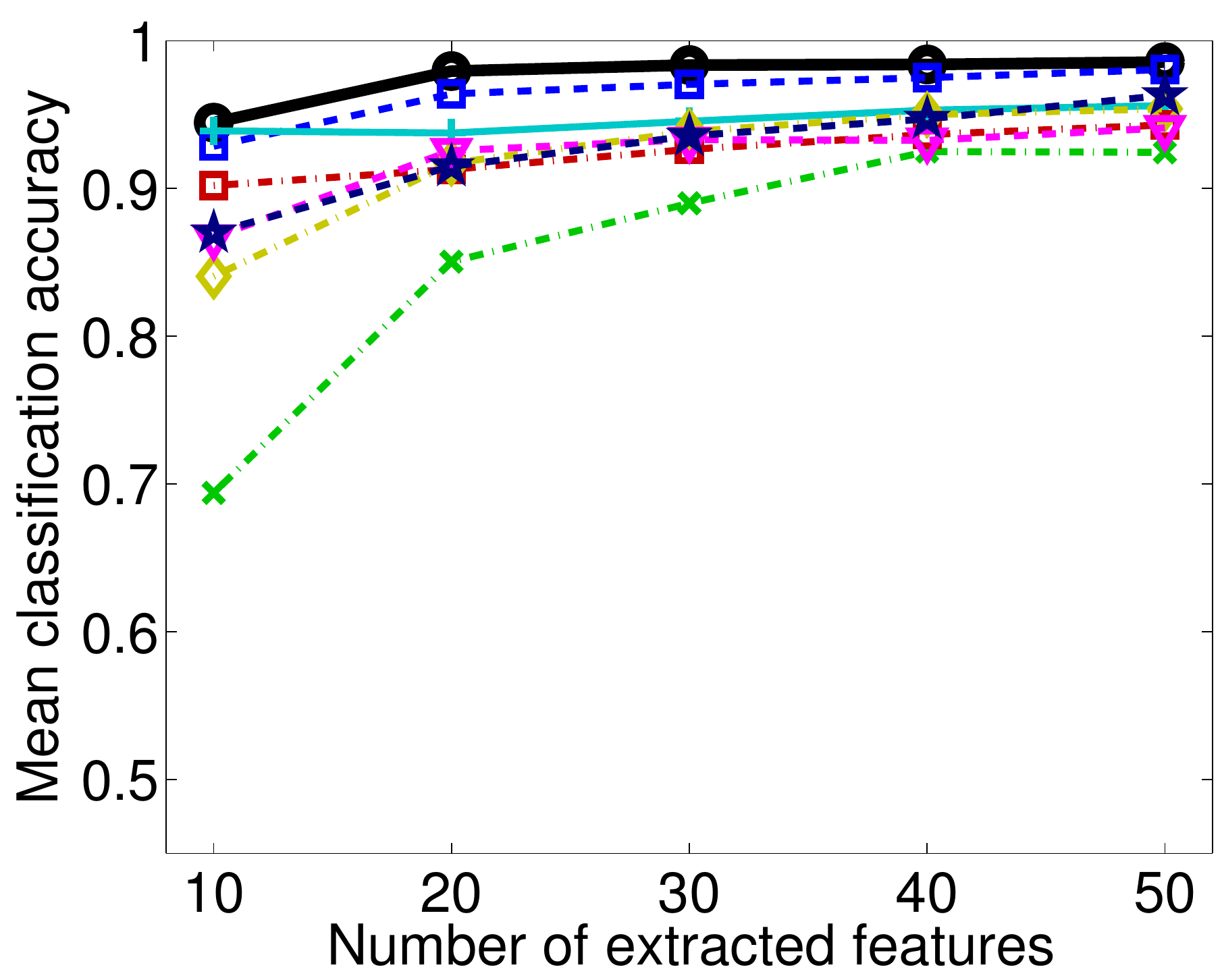}} \\  \vspace{-0.10cm}
(C) PIX10P
  \end{minipage} 
\begin{minipage}[t]{0.245\linewidth}
\centering
  {\includegraphics[width=0.99\textwidth]{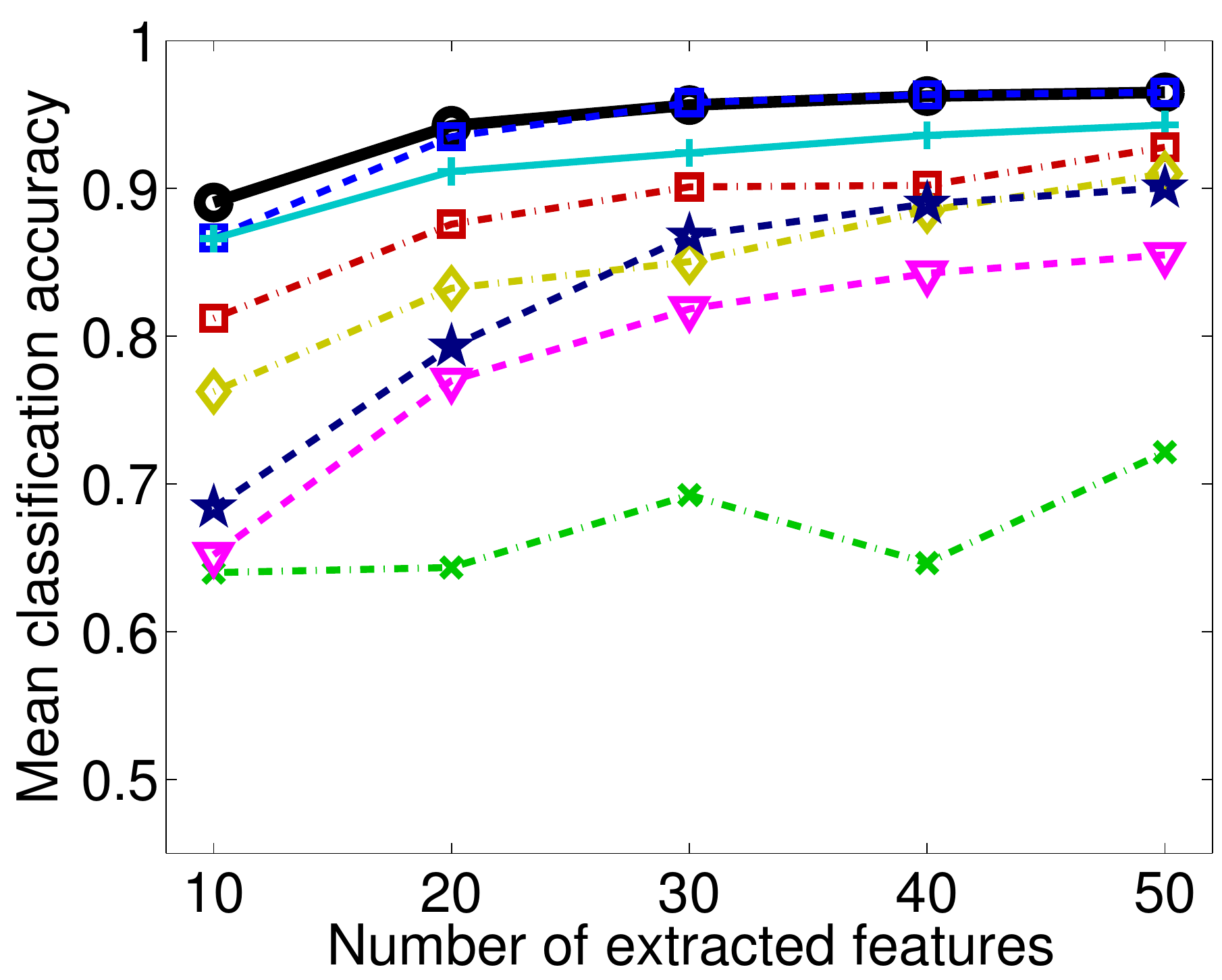}} \\ \vspace{-0.10cm}
(D) ORL10P
\end{minipage}
\begin{minipage}[t]{0.245\linewidth}
\centering
  {\includegraphics[width=0.99\textwidth]{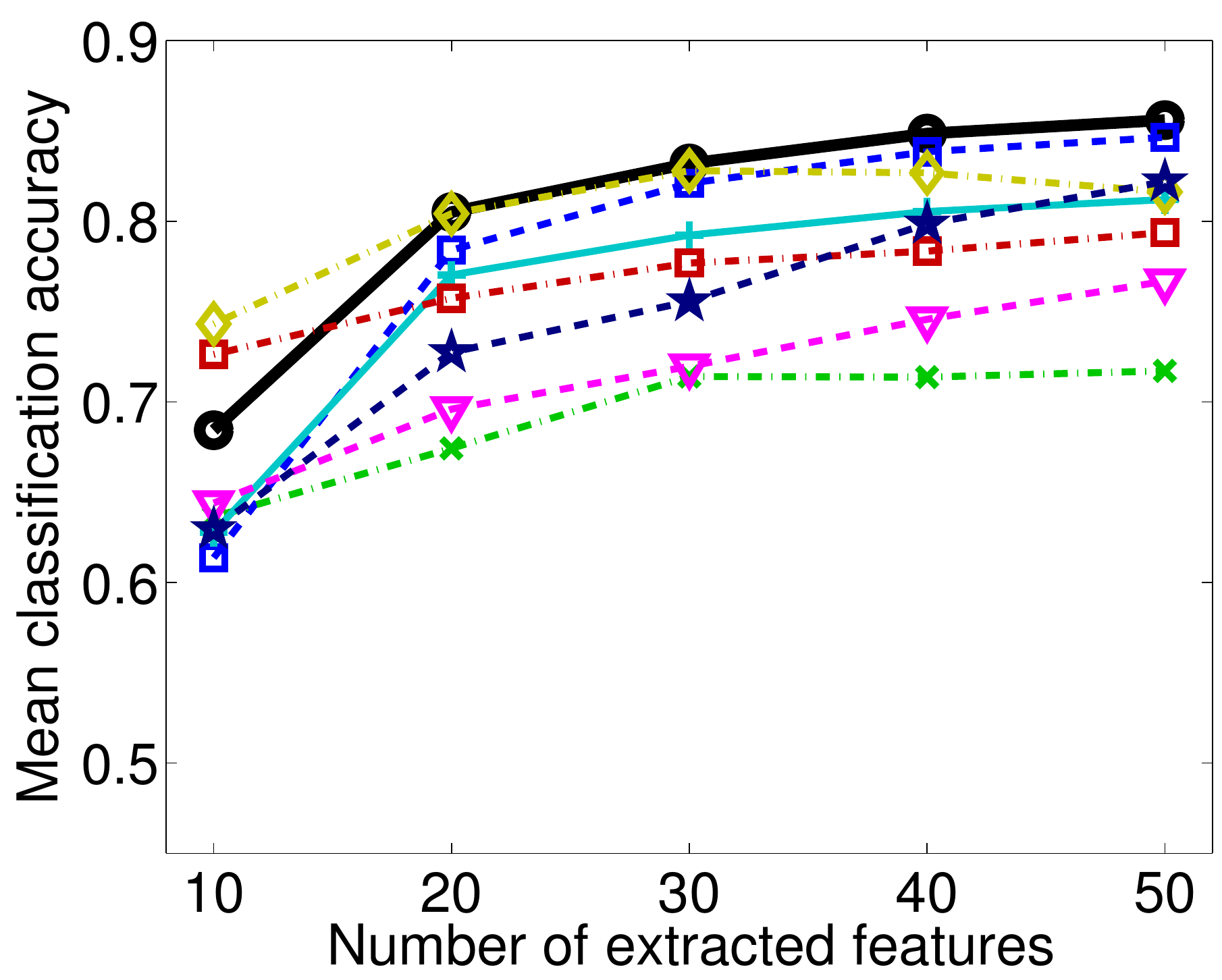}} \\ \vspace{-0.10cm}
(E) TOX
\end{minipage}
  \begin{minipage}[t]{0.245\linewidth}
\centering
{\includegraphics[width=0.99\textwidth]{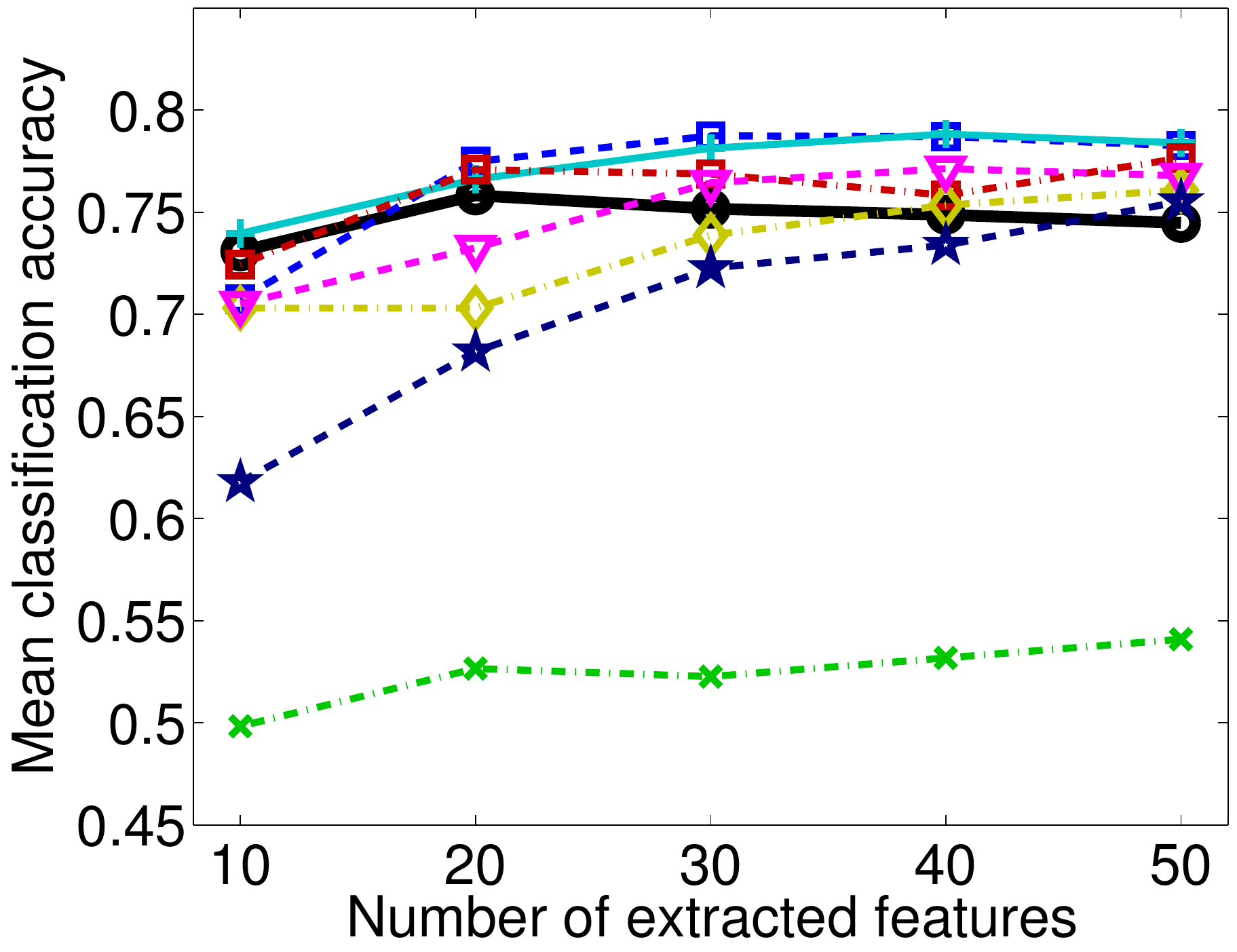}} \\  \vspace{-0.10cm}
(F) CLL-SUB
  \end{minipage}
   \begin{minipage}[t]{0.245\linewidth}
\centering
 {\includegraphics[width=0.99\textwidth]{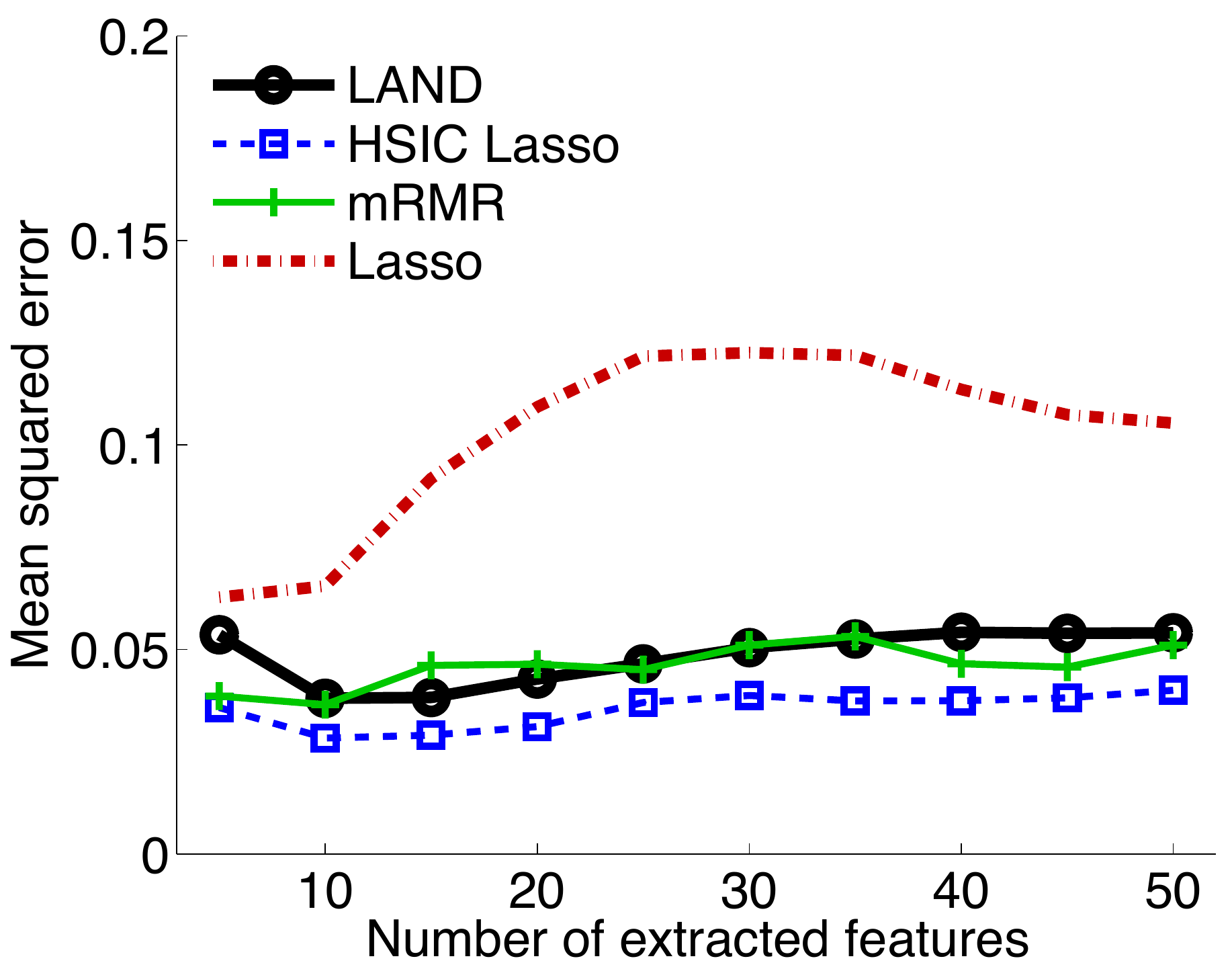}} \\ \vspace{-0.10cm}
(G) TRIM32
  \end{minipage}
 \caption{The results for the benchmark datasets. (A)-(F): Mean classification accuracy for six classification benchmark datasets. The horizontal axis denotes the number of selected features, and the vertical axis denotes the mean classification accuracy. (G): Mean squared error for the TRIM32 data. The horizontal axis denotes the number of selected features, and the vertical axis denotes the mean squared error (lower is better).} 
    \label{fig:result_ASU}
\end{center}
\end{figure*}

\section{Experiments}
\label{sec:experiments}

We first illustrate LAND on a synthetic data and small-scale benchmark datasets. Then, we evaluate LAND using three biological datasets with $d$ ranging from thousands to over a million features. 

\subsection{Evaluation metrics}
We employ the average area under the ROC curve (\emph{AUC}) as a measure of accuracy that is robust with respect to unbalanced classes~\cite{AUC}. Values above 0.5 indicate better-than-random performance; one signifies perfect accuracy.

Let us also define the \emph{dimensionality reduction rate} as $1 - \frac{m}{d}$ where zero represents the original full set of features and higher values indicate smaller sets of selected features. 

Finally, to check whether an algorithm can successfully select non-redundant features, we define the   \emph{independence rate}: 
\begin{align*}
I = 1 - \frac{1}{m(m-1)} \sum_{\boldu_k, \boldu_l, k > l} |\rho_{k,l}|,
\end{align*}
 where $\rho_{k,l}$ is the Pearson correlation coefficient between the $k$-th and $l$-th features. A large $I$ means that the selected features are more independent, or less redundant. In fact, $I$ is closely related to the redundancy rate~\cite{zhao2010efficient}.

\subsection{Synthetic dataset}

We consider a regression problem from a 2000-dimensional input, where input data $(X_1, \ldots, X_{2000})$ includes three groups of variables. The first group of three variables $(X_1, X_2, X_3)$ are relevant for the output $Y$, which is generated according to the following expression:
\[
Y = X_{1}*\exp(X_{2}) + X_{3} + 0.1*E,
\] 
where $E \sim N(0,1)$.  All variables are normally distributed. In particular, for the first 1000 variables, $(X_1, \ldots, X_{1000})^\top \sim N(\boldzero_{1000},\boldI_{1000})$. Here, $N(\boldmu,\boldSigma)$ denotes the multi-variate Gaussian distribution with mean $\boldmu$ and covariance matrix $\boldSigma$. We define the remaining 1000 variables as: $X_{1001} = X_1 + 0.01*E, \ldots, X_{2000} = X_{1000} + 0.01*E$. The second group of variables $(X_4, \ldots, X_{1000})$ and $(X_{1004}, \ldots, X_{2000})$ are uncorrelated with the output, and therefore irrelevant. The third group $X_{1001}$, $X_{1002}$, and $X_{1003}$ are redundant features of $X_1$, $X_2,$, and $X_3$, respectively.  

Figure~\ref{fig:illustrative_example}(a) shows the regularization path for 10 features, and this illustrates that LAND can select non-redundant features. Figure~\ref{fig:illustrative_example}(b) plots the computational time for LAND on a Xeon 2.4GHz (16 cores) with 24GB memory. As can be seen, the computational cost of LAND without using the Nystr\"{o}m approximation and distributed computing increases dramatically with the number of observations. Moreover, since LAND needs $O(dn^2)$ memory space, it is not possible to solve LAND even if the number of observations is small ($n=1000$). Thus, the Nystr\"{o}m approximation and distributed computation are necessary for the proposed method to solve high-dimensional and large sample cases. 

\subsection{Benchmark datasets}

Here, we evaluate the accuracy of LAND using real-world benchmark
datasets (see Table~\ref{tab:feat_data_bench} for details).  

\subsubsection{Classification}


\begin{table}[t]
\centering
\caption{Summary of benchmark datasets.}
\label{tab:feat_data_bench}
\begin{tabular}{clcc}
\hline
Type & Dataset & Features ($d$) & Samples ($n$)  \\ \hline
 & AR10P    & 2400    & 130  \\
& PIE10P   & 2400    & 210  \\
Classification& PIX10P   & 10000   & 100   \\
& ORL10P   & 10000   & 100   \\ 
 & TOX      & 5748    & 171  \\
& CLL-SUB  & 11340   & 111   \\  \hline
Regression & TRIM32 & 31098 & 120 \\ \hline
\end{tabular}
\end{table}

We first consider classification benchmarks with relatively small $d$ and $n$,\footnote{\url{http://featureselection.asu.edu/datasets.php}} allowing us to compare LAND with several baseline methods that are computationally slow. For these classification experiments, we use 80\% of samples for training and the rest for testing. 
In each experiment, we apply feature selection on the training data to select the top $m = 10, 20, \ldots, 50$ features and then measure accuracy using the selected features in the test data. 
We run the classification experiments 100 times by randomly selecting training and test samples and report the average classification accuracy. 
Since all datasets are multi-class, we use multi-class kernel logistic regression (KLR)~\cite{yamada2010semi}. 
For KLR we use a Gaussian kernel where the kernel width and the regularization parameter are chosen based on 3-fold cross-validation.

Figures~\ref{fig:result_ASU}(A)-(F) show the average classification accuracy versus the number of selected features.  With the single exception of the CLL-SUB benchmark, LAND compares favorably with all baselines, including HSIC Lasso, a state-of-the-art high-dimensional feature selection method.  

\subsubsection{Regression}

Our last benchmark is the Affymetric GeneChip Rat Genome 230 2.0 Array dataset~\cite{scheetz2006regulation}. 
In this dataset, there are 120 rat subjects ($n=120$). The real-valued expression for over 30 thousand genes ($d=31098$) in the mammalian eye is important for eye disease. In this paper, we focus on finding genes that are related to the TRIM32 gene~\cite{scheetz2006regulation,huang2010variable}, 
which was recently found to cause the Bardet-Biedl syndrome. 

For this regression experiment, we use 80$\%$ of samples for training and the rest for testing. 
We again select  the top $m = 10, 20, \ldots, 50$ features having the largest absolute regression coefficients in the training data. 
As earlier, we run the regression experiments 100 times by randomly selecting training and test samples, and compute the average mean squared error. We employ  kernel regression~\cite{book:Schoelkopf+Smola:2002}
with the Gaussian kernel. 
The Gaussian width and the regularization parameter are chosen based on 3-fold cross-validation.  
In this experiment, most existing methods are too slow to finish. Thus, we only include the LAND, HSIC Lasso, linear Lasso, and  mRMR results.

Figure~\ref{fig:result_ASU}(G) shows the mean squared error over 100 runs as a function of the number of selected features. As can be observed, the accuracy obtained with features selected by LAND is better than Lasso, comparable with  mRMR, and slightly worse than HSIC Lasso. 
This is because in this regression experiment, the HSIC measure of independence performs better than NHSIC. If we use HSIC instead of NHSIC, LAND can achieve the same accuracy as HSIC Lasso. 

\begin{figure}[t]
\centerline{\includegraphics[width=0.5\columnwidth]{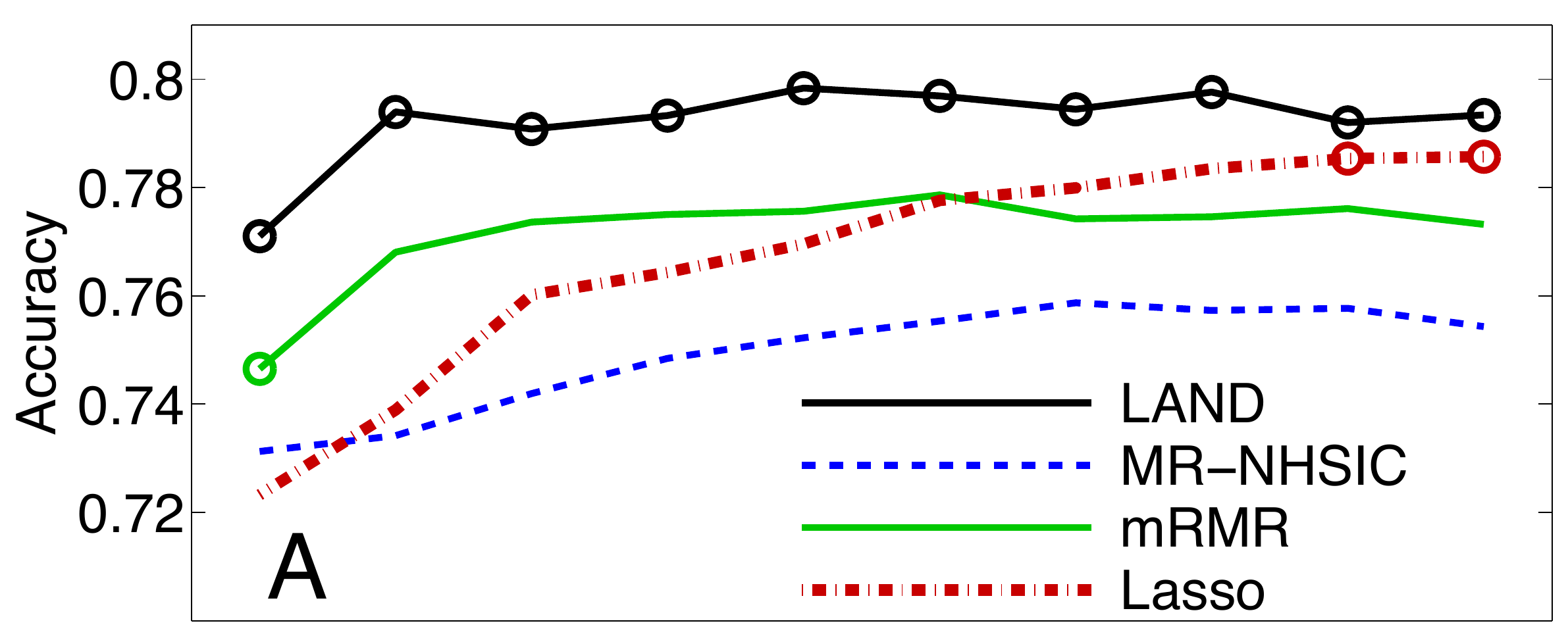}}  
\vspace{-0.2em}
\centerline{\hspace{-0.02cm}\includegraphics[width=0.5\columnwidth]{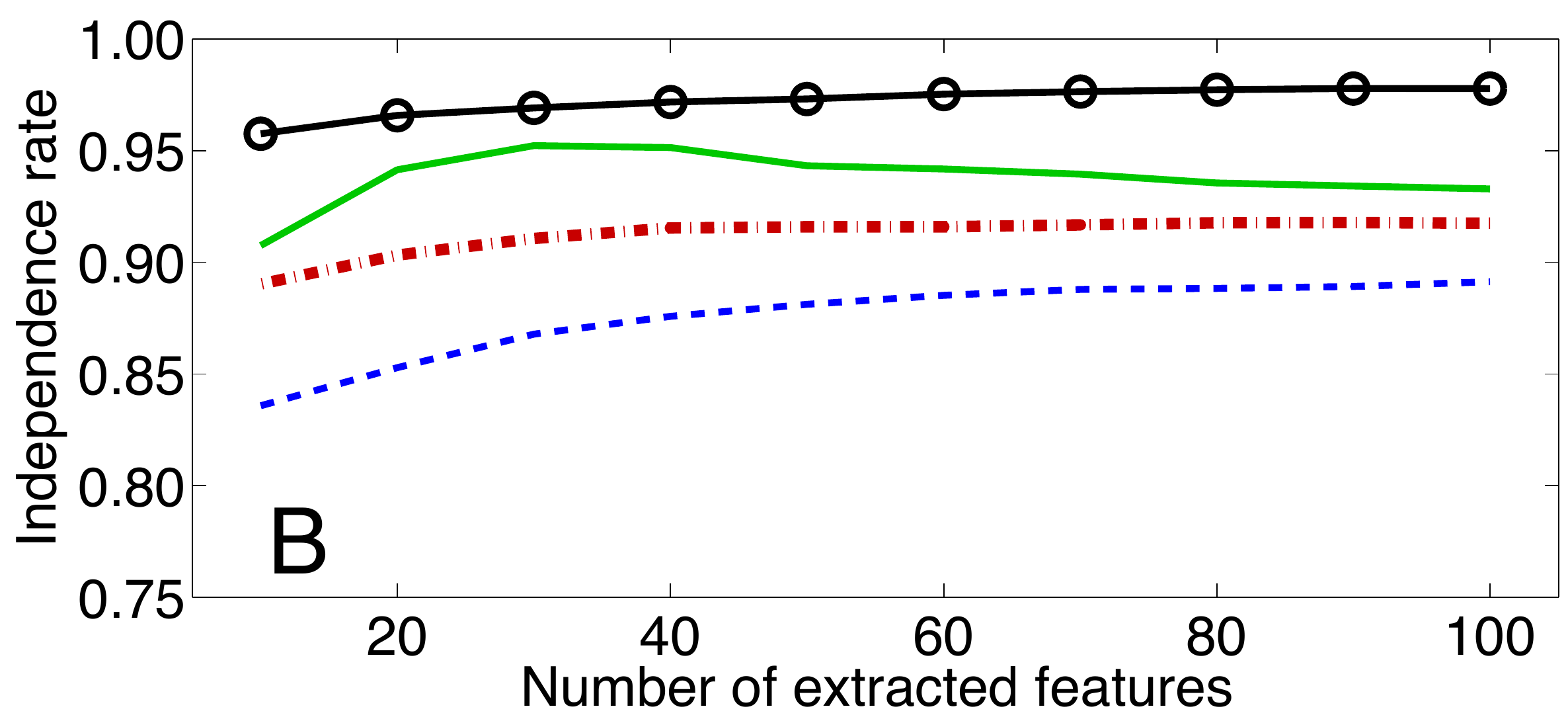}} 
\caption{Results on the p53 benchmark. (A)~Accuracy vs. number of selected features $m$. Circles indicate the methods achieving the best average accuracy according to a one-tailed t-test ($p<0.05$). Differences between LAND and all baselines are statistically significant for all dimensionality reduction levels except $m=10$ (vs. mRMR) and $m \geq 90$ (vs. Lasso). (B)~Independence rate vs. $m$: all differences are significant as indicated by circles.
\label{fig:lars_p53}
}
\end{figure}

\subsection{High dimensional and large-scale datasets}
In this section, we evaluate LAND using real-world high-dimensional and/or large-scale datasets.

We evaluate all feature selection methods by passing the selected features to a supervised learning algorithm. For this purpose we employ gradient boosting decision trees (GBDT)~\cite{friedman2001greedy} as an off-the-shelf nonlinear classifier.
\subsubsection{Prediction of p53 transcriptional activities} 
\label{sec:p53}

\begin{table}
\centering
\caption{Sizes of biological datasets and computational times (in seconds) to select 100 features using the nonlinear methods.\label{tab:feat_data_bio}}
\begin{tabular}{lcccccc}
\hline
Dataset & $d$ & $n$ &  \small{MR-NHSIC} & \small{mRMR} & \small{LAND}\\ \hline
p53 & 5408 & 26120 & 290 & 544 & 1709\\ 
PC & 276322 & 302 & 383 & 4018 & 1284\\
Enzyme & 1062420 & 13794 & 5328 & n/a & 10630\\ \hline
\end{tabular}
\end{table}

We first consider the p53 mutant dataset~\cite{danziger2009predicting}, where the goal is to predict whether any of $n=31420$ mutations is active or inactive based on $d=5408$ features. Class labels are obtained via \emph{in vivo} assays~\cite{danziger2009predicting}. For this data, we compared LAND with Lasso, mRMR, and MR-NHSIC baselines (see Methods) on the task of selecting $m=100$ features. We present results for LAND using $b = 20 \ll n$ and setting the basis vector heuristically to $\boldu_b = [-5, -4.47, \ldots, 4.47, 5.0]^\top \in \mathbbR^{20}$.  

Accuracy and independence rate metrics for features selected by LAND are compared with those obtained by three state-of-the-art feature selection baselines. We split the data into 26,420 observations used for training the learning algorithm and 5,000 observations for testing.  We run the classification experiments 20 times by randomly selecting training and test samples, and report the average accuracy and independence rate metrics. We select the $m$ most relevant features ($m = 10, 20, \ldots, 100$) and employ 100 trees with 20 nodes in the GBDT classifier. 

\begin{figure}[t]
\centerline{\includegraphics[width=0.5\columnwidth]{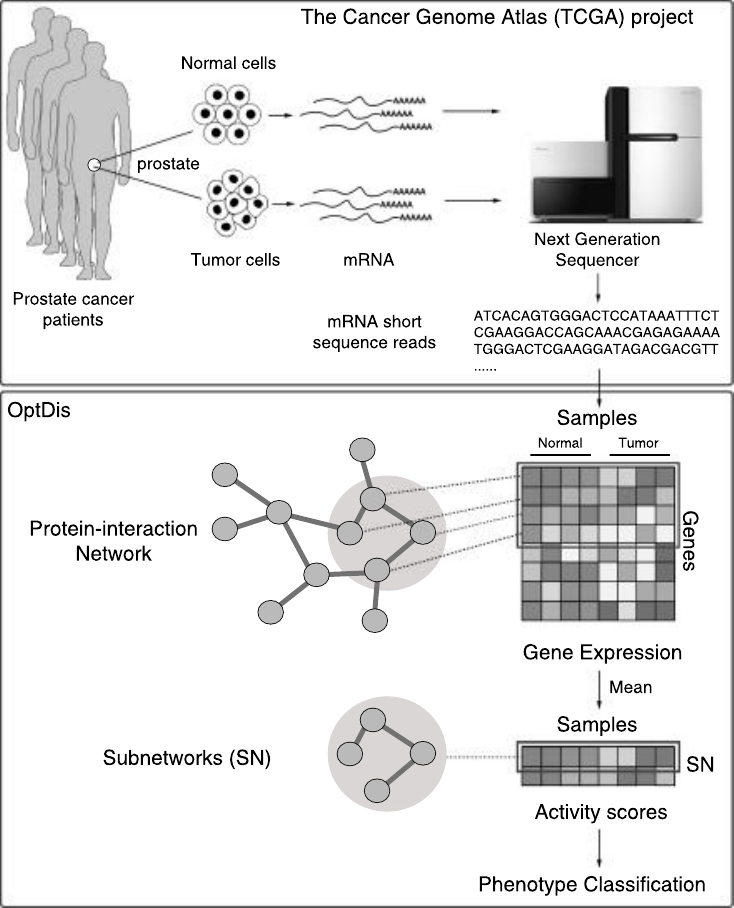}}
\caption{Overview of prostate cancer data. As part of The Cancer Genome Atlas (TCGA) project~\cite{TCGA_PRAD_2015}, mRNA sequence based gene expression profiles of tumor and normal cells were obtained from a cohort of prostate adenocarcinoma patients. The data was downloaded from the TCGA data portal. Using OptDis \cite{OPTDIS_2011}, we extracted connected subnetworks of maximum size 7 from the STRING v10 protein-protein interaction network~\cite{Szklarczyk2014} using only edges with score above 0.9. We finally used the average expression of the component genes of each subnetwork as a feature. 
\label{fig:prostateIntro}
}
\end{figure}

\begin{figure}[t]
\centerline{\includegraphics[width=0.5\columnwidth]{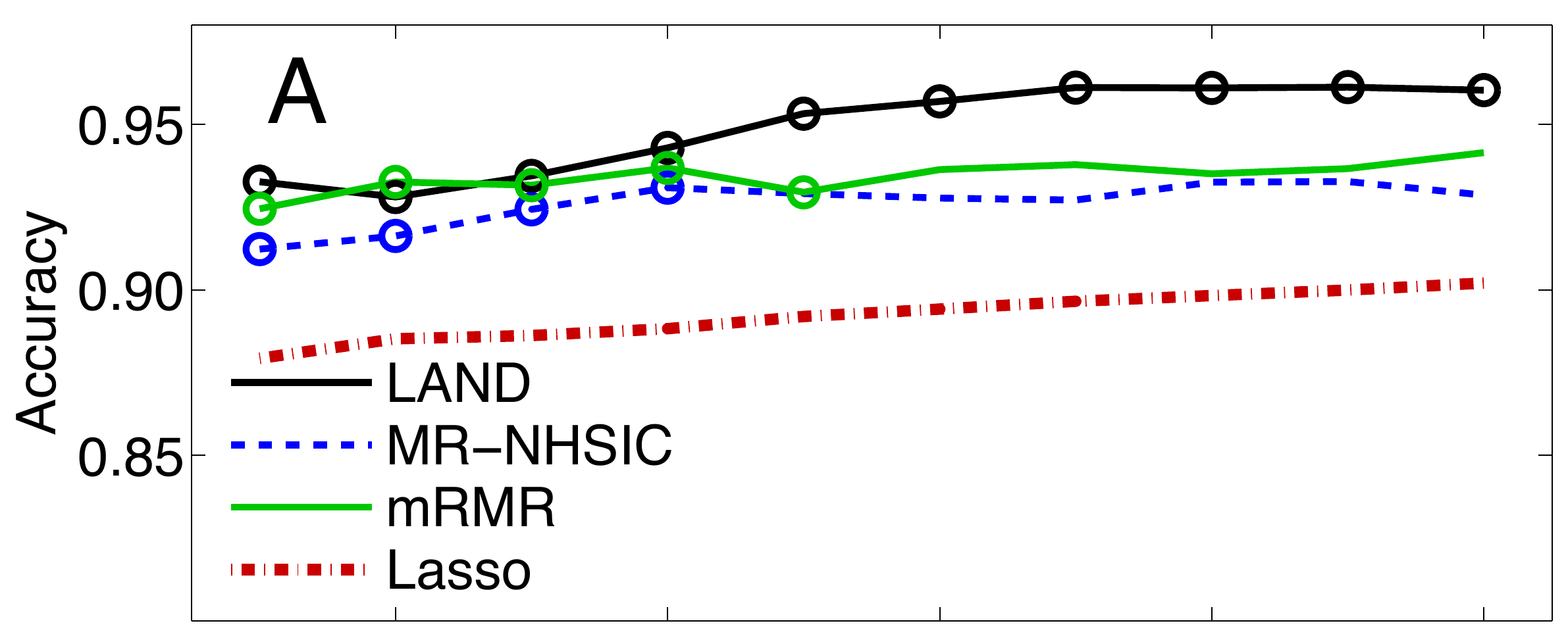}}  
\vspace{-0.2em}
\centerline{\hspace{-0.02cm}\includegraphics[width=0.5\columnwidth]{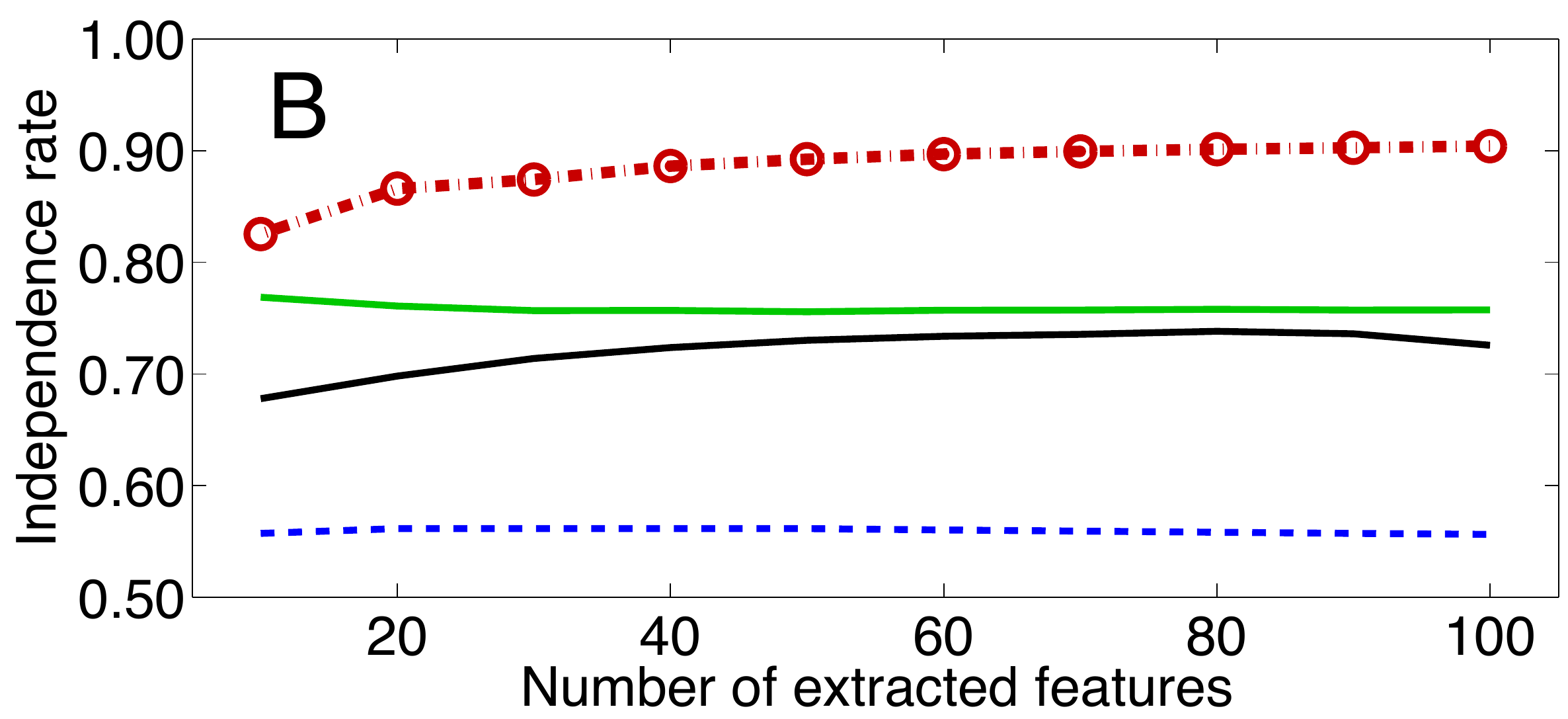}} 
\caption{Results on the prostate cancer dataset.  (A)~Accuracy vs. number of selected features $m$. Circles indicate the methods achieving the best average accuracy according to a one-tailed t-test ($p<0.05$). Differences between LAND and the baselines are statistically significant for $m \geq 60$. (B)~Independence rate vs. $m$: all differences are significant.
\label{fig:lars_cenk}
}
\end{figure}

\begin{figure}[t]
\centerline{\includegraphics[width=0.5\columnwidth]{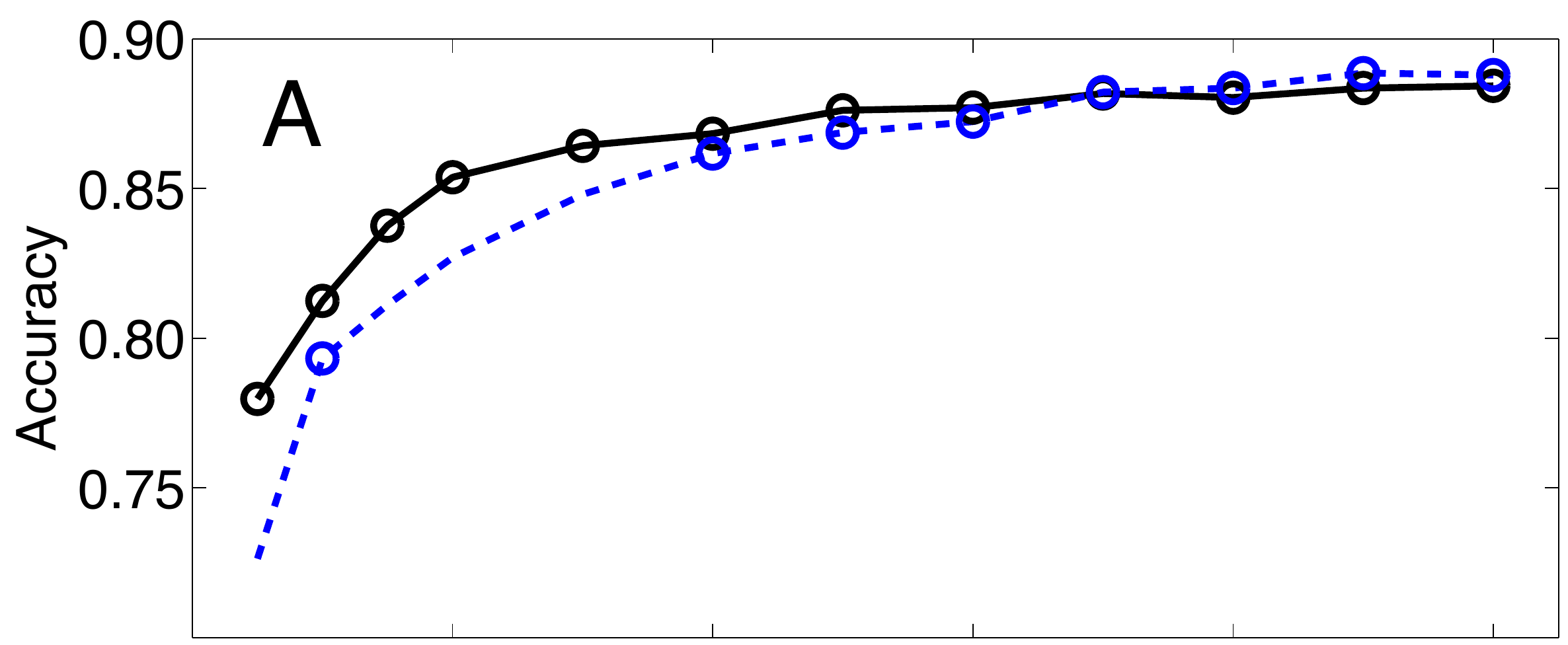}} 
\centerline{\includegraphics[width=0.5\columnwidth]{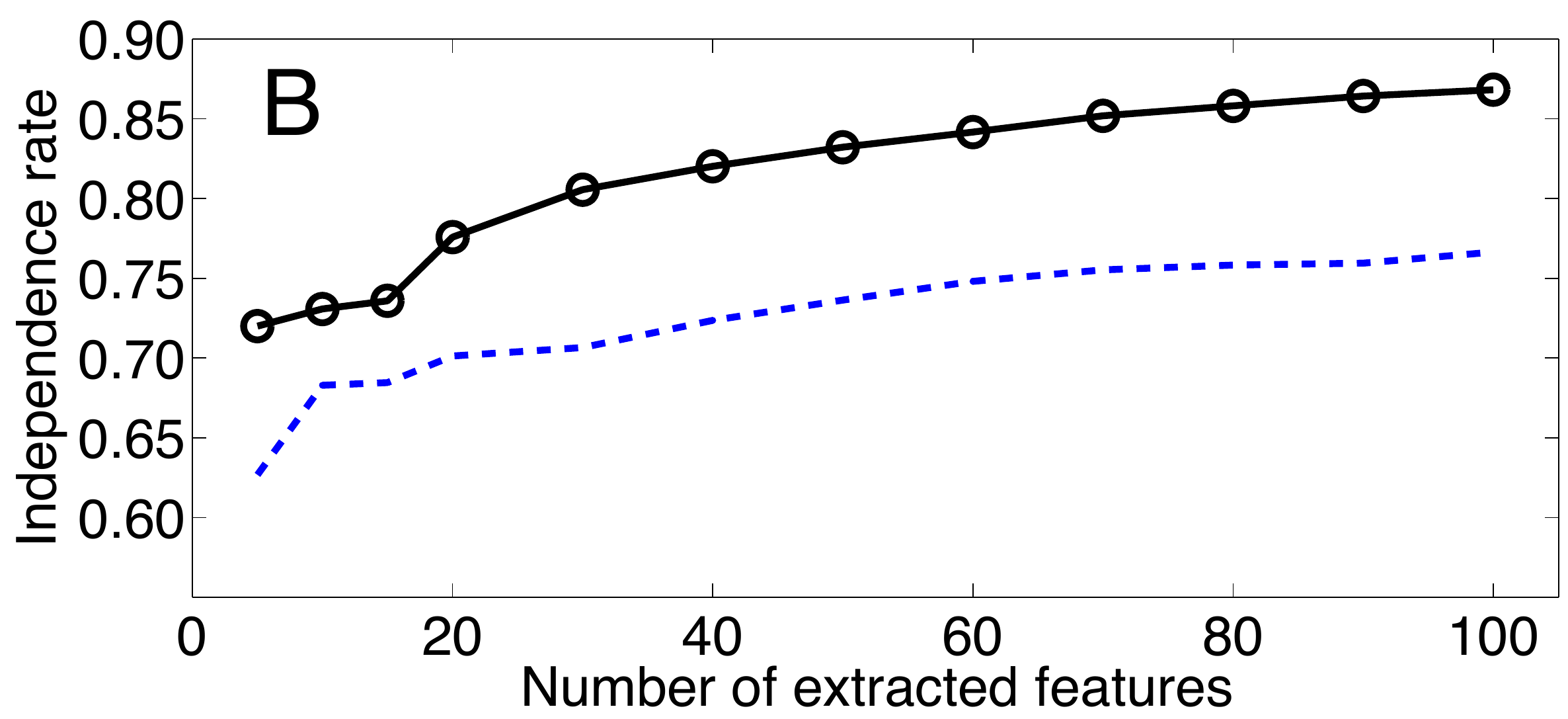}} 
\vspace{-0.2em}
\centerline{\includegraphics[width=0.5\columnwidth]{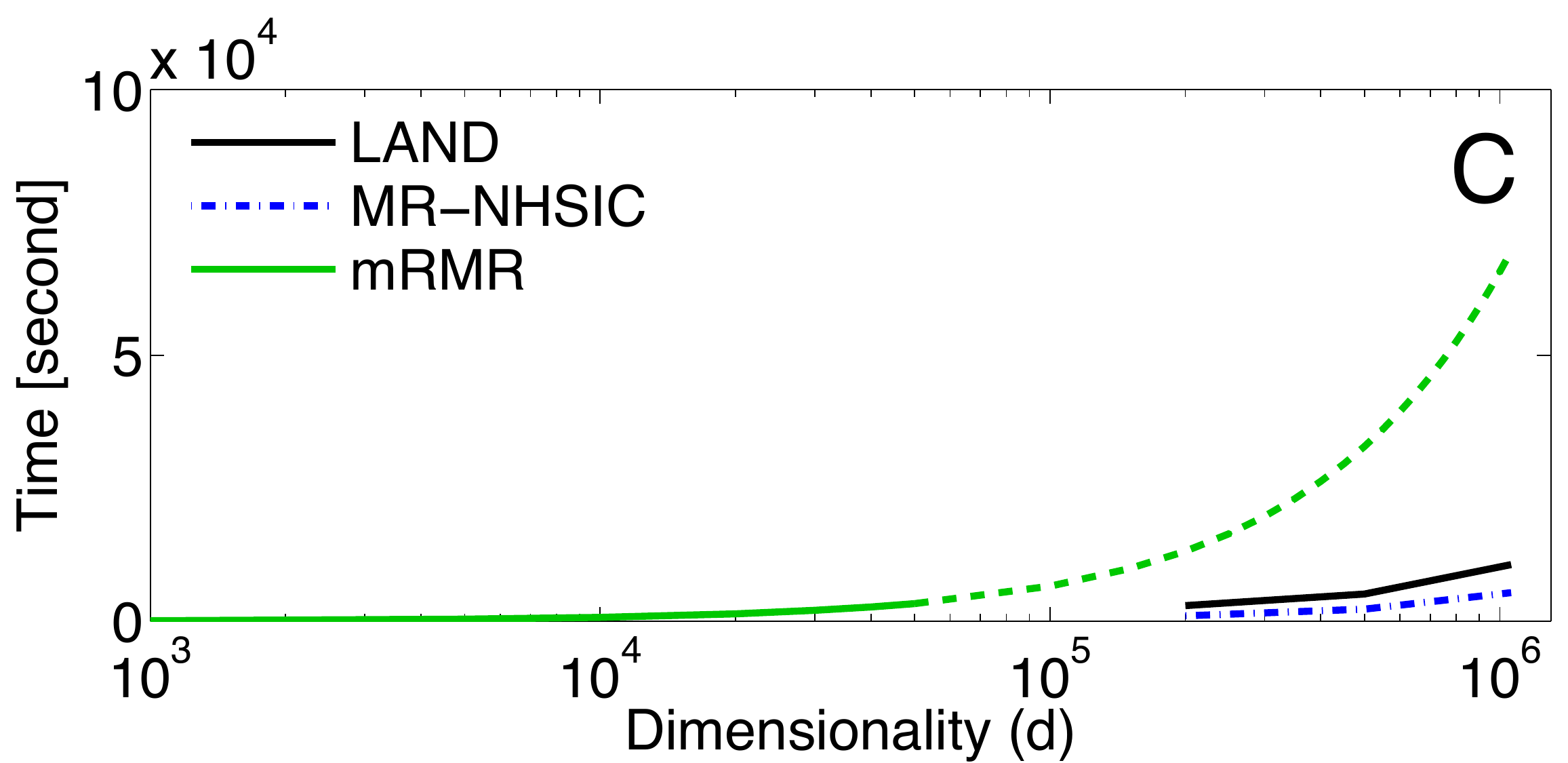}} 
\caption{Results for the enzyme dataset.  (A)~Differences in accuracy are significant for $m=5,20,30,40$. We also ran two linear Lasso baselines (not shown) with regularization parameter $\lambda=10^{-6}$ and $0.5\times 10^{-6}$. These settings yield low accuracy (AUC between 0.75 and 0.85) with many features (an average of $m=$1665 and 7675, respectively). (B)~Independence rate vs. $m$. (C)~Feature selection runtime vs. $d$. We ran the algorithms on reduced-dimensionality datasets. We estimated the mRMR runtime for $d>$ 50,000, given that it scales linearly with $d$ (dotted line).\label{fig:lars_enz}} 
\end{figure}

\begin{figure}[t]
\centerline{\includegraphics[width=0.7\columnwidth]{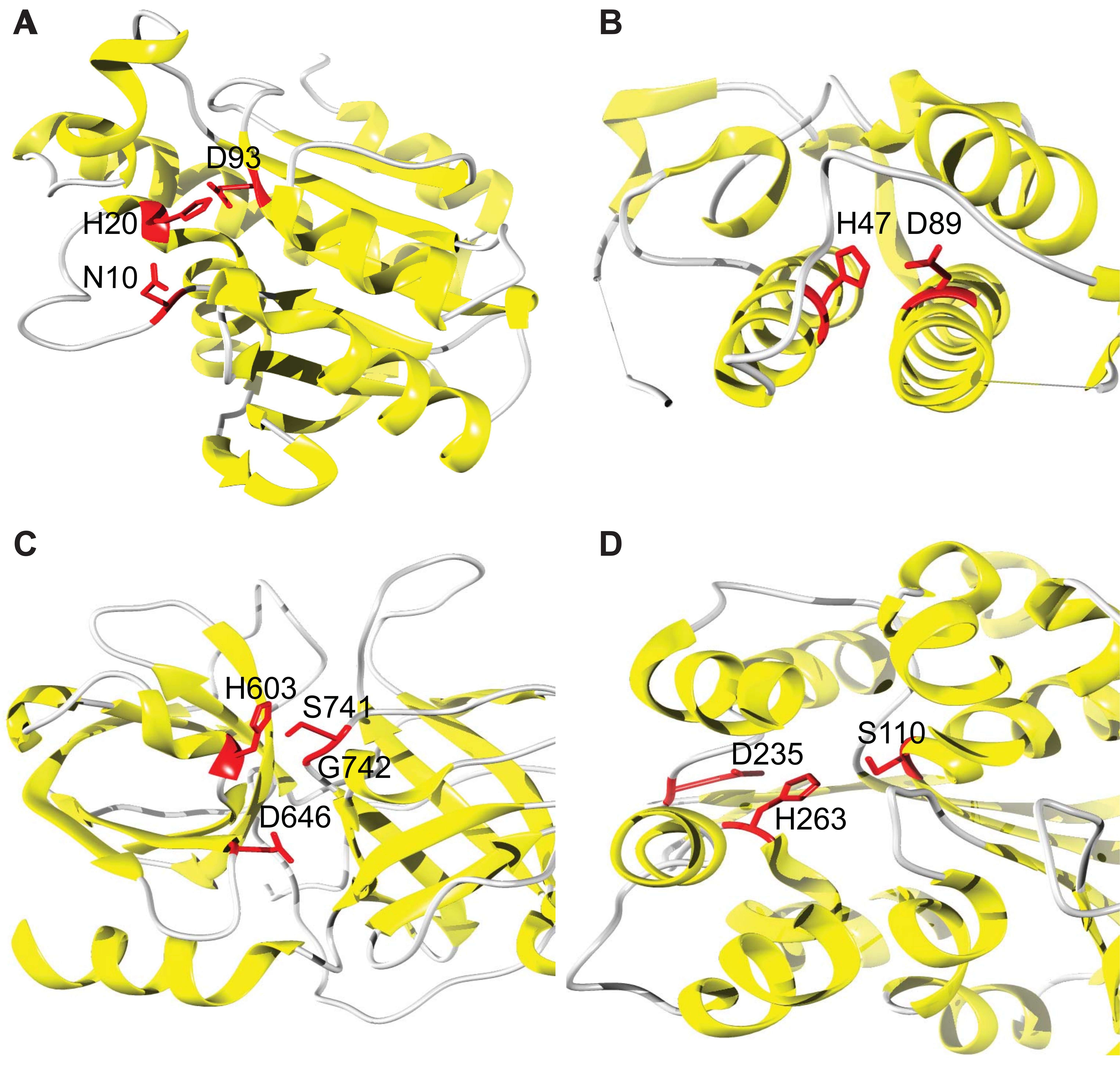}}  
\caption{Structural signatures of enzymes. Illustration of four distinct protein structures from PDB containing the DH structural motif. Catalytic residues, shown in red, were experimentally determined and extracted from Catalytic Site Atlas. (A)~Peptidyl-tRNA hydrolase (chain A of PDB entry 2pth) where residues N10, H20 and D93 form a catalytic triad. (B)~Basic phospholipase A2 piratoxin-3 (chain A of PDB entry 1gmz) with catalytic residues H47 and D89. (C)~Plasminogen (chain A of PDB entry 1qrz) with catalytic residues H603, D646, S741 and G742. (D)~2-hydroxy-6-oxo-6-phenylhexa-2,4-dienoate hydrolase (chain A of PDB entry 1c4x) where residues S110, D236 and H263 form a catalytic triad.\label{fig:cat_sig}}
\end{figure}

Runtimes are shown in Table~\ref{tab:feat_data_bio}. Given the small dimensionality $d$ of this problem, the speed up from distributed feature selection algorithms is not sufficient to offset their cluster overhead. Therefore the single-machine linear algorithm (Lasso) is the most efficient. LAND, however, selects better features. Fig.~\ref{fig:lars_p53}(a) shows that high accuracy (80\% AUC) can be achieved by LAND with a dimensionality reduction of over 99\%. Considering more than $m=20$ features yields marginal improvements in accuracy.  If one selects a small number of features, LAND outperforms the state-of-the-art nonlinear methods in accuracy. Conversely, the same accuracy can be achieved with fewer features. For this small-size benchmark, the performance of the state-of-the-art in linear feature selection (Lasso) is comparable.  
Fig.~\ref{fig:lars_p53}(B) plots the independence of the selected features versus $m$. The features selected by LAND are significantly less redundant compared to the baselines, irrespective of the dimensionality reduction. In summary, LAND selects the most independent features and achieves the top accuracy.  


\subsubsection{Subnetwork markers for prostate cancer classification}
\label{sec:cenk}

Next, we applied our approach to a cohort of $n=383$ prostate cancer (PC) patients. We aim to separate malignant tumors from normal tissues based on $d = 276322$ features (Fig.~\ref{fig:prostateIntro}). For LAND, we used $b = 20 \ll n$ and $\boldu_b = [-5, -4.47, \ldots, 4.47, 5.0]^\top \in \mathbbR^{20}$.

Accuracy and independence rate of features selected by LAND are compared with those obtained by the three baselines. We split the data into 344 patients used for training the learning algorithm and 39 patients for testing. We ran the classification experiments 20 times by randomly selecting training and test observations, and report average performance. We select $m$ most relevant features ($m = 10, 20, \ldots, 100$). The classifier employs 100 trees with 20 nodes. 


As shown in Fig.~\ref{fig:lars_cenk}(A), LAND achieves the best accuracy (AUC above 95\%) with as few as $m=80$ features --- a dimensionality reduction over 99.97\%. Lasso selects more independent features (Fig.~\ref{fig:lars_cenk}(B)), however its accuracy is lower.  
As shown in Table~\ref{tab:feat_data_bio}, LAND is more efficient than mRMR on a cluster. The time required  by mRMR increases dramatically with the dimensionality of the data, while the computing time for LAND does not. MR-NHSIC is even faster, however it does not select independent features (Fig.~\ref{fig:lars_cenk}(B)).

\subsubsection{Enzyme protein structure detection} 
\label{sec:enzyme}


Our third task is to distinguish between enzyme and non-enzyme protein structures. Enzyme data contains all homomeric protein structures from the Protein Data Bank (PDB) \cite{Berman2000} as of February 2015 such that (i) each structure is at least 50 amino acid residues long; and (ii) protein structure is determined by X-ray crystallography with resolution below 2.5\AA. All proteins with 100\% sequence identity to any other protein in the dataset are filtered out. In the case of multiple exact matches, the structure with best resolution is selected. To generate the features, the protein structures are modeled as contact graphs, where each residue is represented as a labeled vertex and two spatially close residues, with Euclidean distance between any two atoms below 4.5\AA, are linked by an undirected edge. Each feature is obtained by counting labeled rooted graphlets with up to four vertices \cite{Lugo2014}. 
More details on rooted graphlets can be found in the literature \cite{Przulj2004, Przulj2007, vacic2010graphlet}.
There are $n=15328$ observation vectors (7,767 enzymes and 7,561 non-enzymes), each of dimensionality $d=1062420$. We split the observations into 90\% (13794) for training and 10\% (1534) for testing. We report average performance across five random splits. We use classifiers with 500 trees and 20 nodes, and select $m = 5,10,20,\ldots,100$ features.  Circles indicate best accuracy/independence rate according to t-tests ($p<0.05$).

To explore the complexity stemming from the ultra-high dimensionality of this problem, we trained a state-of-the-art classifier based on the graph kernel method \cite{Lugo2014} on the full dataset. The resulting model achieved high accuracy (AUC above 90\%), but needed roughly 18 days of computing time using our fastest server (a machine with 64 2.4~GHz processors and 512~GB of RAM). This demonstrates the need for feature selection. 

The naive version of LAND requires memory that scales as as $O(dn^2)$ (cf. Table~\ref{tab:complexity}) for storing the kernel matrices. Due to the very large number of features and the large number of observations in this dataset, this is prohibitive --- approximately 1.5 petabytes. By using $b=10 \ll n$, we reduce the space complexity to $O(dbn)$ (cf. Table~\ref{tab:complexity}) and the memory requirement to a more manageable one terabyte for LAND. We also use $\boldu_b = [-5, -3.89, \ldots, 3.89, 5.0]^\top \in \mathbbR^{10}$.

Due to the size of this dataset, running the mRMR baseline would require hundreds of gigabytes of memory. Since this is unfeasible, we only compare LAND with one nonlinear baseline (MR-NHSIC). LAND achieves higher accuracy when selecting very few features, reducing the dimensionality of the problem by over 99.99\% (Fig.~\ref{fig:lars_enz}(A)).  This also implies more interpretable results and an enormous speed up in classification/prediction time. Accuracy is only slightly lower than what is obtained with the state-of-the-art classifier using all $d$ features. Two linear (Lasso) baselines yielded lower accuracy with a worse dimensionality reduction. The features selected by LAND have lower redundancy than those selected by the baseline (Fig.~\ref{fig:lars_enz}(B)).
%
Although Lasso does not achieve good performance, it is the most efficient and can select features in 973 seconds.  Fig.~\ref{fig:lars_enz}(C) compares the computational time of LAND, MR-NHSIC, and mRMR. LAND runs in just 3 hours on a computer cluster --- about double the MR-NHSIC runtime (cf. Table~\ref{tab:feat_data_bio}). This is a trade-off for the independence of the selected features. 

\section{Discussion}
\label{sec:conclusion}

The proposed LAND feature selection method is guaranteed to find a globally optimal solution, and does so efficiently by exploiting a non-negative variant of the LARS algorithm where the parameter space is sparse. Our experimental results demonstrate that LAND scales up to large and high-dimensional biological data by allowing a distributed implementation on commodity cloud computing platforms.

Let us further investigate the structural motifs identified by LAND from the enzyme data. Our initial requirement was that the method be able to identify catalytic site-related features, as well as other features relevant for the signatures of enzymatic structures. One of the most interesting motifs is the DH graphlet, containing aspartic acid and histidine, which is also the most commonly seen motif in the Catalytic Site Atlas \cite{Porter2004}. Combined with a variety of other residues, such as serine and asparagin, the DH motif frequently forms catalytic sites.  Fig.~\ref{fig:cat_sig} shows four different protein structures with an identified DH motif. LAND is therefore able to identify biologically-relevant motifs in extremely large feature spaces and can be readily used to speed up structure-based models in computational biology. It could also enhance data exploration, e.g., via a recursive study of all DH enzymes, which would result in subtyping of catalytic sites. While identification of structural motifs is not novel in computational biology \cite{Xin2011}, scaling up such methods to extremely high dimensions is an important new step in the field.







\bibliographystyle{IEEEtran}
\bibliography{main}

%








\end{document}